\documentclass[letterpaper]{article} 
\usepackage{aaai24}  
\usepackage{times}  
\usepackage{helvet}  
\usepackage{courier}  
\usepackage[hyphens]{url}  
\usepackage{graphicx} 
\usepackage{natbib}  
\usepackage{caption} 
\usepackage{algorithm}
\usepackage{algorithmic}

\usepackage{amsmath}
\usepackage{amsthm}
\usepackage{amssymb}
\usepackage{diagbox}

\newtheorem{definition}{Definition}
\newtheorem{proposition}{Proposition}

\usepackage{enumitem}
\usepackage{subfigure}

\setlist[enumerate]{nosep}
\setlist[itemize]{nosep}

\usepackage{newfloat}
\usepackage{listings}
\DeclareCaptionStyle{ruled}{labelfont=normalfont,labelsep=colon,strut=off} 
\floatstyle{ruled}
\newfloat{listing}{tb}{lst}{}
\floatname{listing}{Listing}
\title{TNPAR: Topological Neural Poisson Auto-Regressive Model for Learning Granger Causal Structure from Event Sequences}
\author{
    Yuequn Liu\textsuperscript{\rm 1},
    Ruichu Cai\textsuperscript{\rm 1,2}\thanks{Corresponding author.},
    Wei Chen\textsuperscript{\rm 1},
    Jie Qiao\textsuperscript{\rm 1},
    Yuguang Yan\textsuperscript{\rm 1},
    Zijian Li\textsuperscript{\rm 1,3},
    Keli Zhang\textsuperscript{\rm 4},
    Zhifeng Hao\textsuperscript{\rm 5}
}
\affiliations{
    \textsuperscript{\rm 1}School of Computer Science, Guangdong University of Technology, Guangzhou, China\\
    \textsuperscript{\rm 2}Peng Cheng Laboratory, Shenzhen, China\\
    \textsuperscript{\rm 3}Machine Learning Department, Mohamed bin Zayed University of Artificial Intelligence, Abu Dhabi, United Arab Emirates\\
    \textsuperscript{\rm 4}Huawei Noah’s Ark Lab, Huawei, Paris, France\\
    \textsuperscript{\rm 5}College of Science, Shantou University, Shantou, China\\

    \{weanyq, cairuichu, chenweidelight, qiaojie.chn\}@gmail.com,
ygyan@gdut.edu.cn,
leizigin@gmail.com,
zhangkeli1@huawei.com,
haozhifeng@stu.edu.cn
}

\usepackage{bibentry}

\begin{document}

\maketitle

\begin{abstract}
Learning Granger causality from event sequences is a challenging but essential task across various applications. Most existing methods rely on the assumption that event sequences are independent and identically distributed (i.i.d.). However, this i.i.d. assumption is often violated due to the inherent dependencies among the event sequences. Fortunately, in practice, we find these dependencies can be modeled by a topological network, suggesting a potential solution to the non-i.i.d. problem by introducing the prior topological network into Granger causal discovery. This observation prompts us to tackle two ensuing challenges: 1) how to model the event sequences while incorporating both the prior topological network and the latent Granger causal structure, and 2) how to learn the Granger causal structure. To this end, we devise a unified topological neural Poisson auto-regressive model with two processes. In the generation process, we employ a variant of the neural Poisson process to model the event sequences, considering influences from both the topological network and the Granger causal structure. In the inference process, we formulate an amortized inference algorithm to infer the latent Granger causal structure. We encapsulate these two processes within a unified likelihood function, providing an end-to-end framework for this task. Experiments on simulated and real-world data demonstrate the effectiveness of our approach.
\end{abstract}

\section{Introduction}
Causal discovery from multi-type event sequences is important in many applications. In the realm of intelligent operation and maintenance, identifying the causal structure behind alarm sequences can expedite the location of root causes \cite{vukovic2022causal}. In social analysis, recovering the causal structure behind users' behavior sequences can inform the development of effective advertising strategies \cite{chen2020mining}.

Lots of methods have been proposed for causal discovery from multi-type event sequences. They fall into two main categories: 1) constraint-based methods, which utilize independence-based tests or measures to estimate the causal structure, and 2) point process-based methods, which employ point processes to model the generation process of event sequences. While the former, including methods such as PCMCI (PC with Momentary Conditional Independence test) \cite{runge2020discovering} and transfer entropy-based methods \cite{chen2020mining,mijatovic2021information}, is founded on strict assumptions of the causal mechanism assumptions, the latter, like Hawkes process-based methods \cite{xu2016learning,zhou2013learning} and neural point process models \cite{Shchur2021Neural}, is based on the concept of Granger causality \cite{granger1969investigating}. Given the weaker assumption of Granger causality, identifying the Granger causal structure just by assessing if a sequence is predictive of another, point process-based methods often find preference in real-world applications \cite{shojaie2022granger} and are therefore the focal point of this work.  

\begin{figure}[t]
    \centering
    \includegraphics[width=0.47\textwidth]{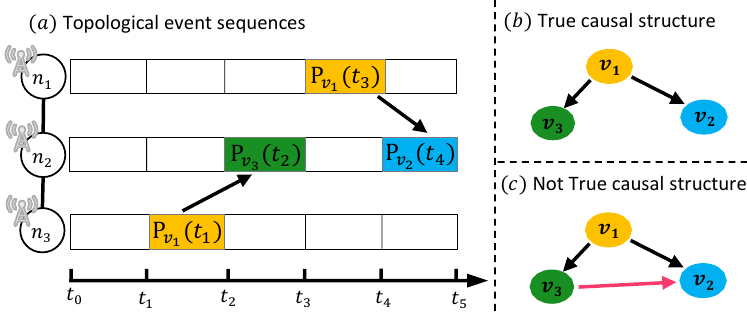}
    \caption{An example of the topological event sequences generated by a mobile network in an operation and maintenance scenario. In this context, $\{n_1, n_2, n_3\}$ represent the network elements, and $\{v_1,v_2,v_3\}$ denote the event types of the alarms. The term $\mathrm{P}_{v_1}(t_1)$ stands for the distribution of $v_1$ alarm at $t_1$ timestamp, and so on. Black arrows represent the correct causal edges, whereas the red arrow represents an incorrect edge.
    (a) Generation process of topological event sequences. The event distribution varies with time and is influenced by both a topological network (connected by the solid lines) and a Granger causal structure (depicted in part (b)).
    (b) Ground truth of the Granger causal structure.
    (c) Granger causal structure learned under i.i.d. assumption.
    }
    \label{fig: Illustrated example for mobile network}
\end{figure}

However, these existing point process-based methods are significantly dependent on the i.i.d. assumption to model the generation process of event sequences. In reality, this assumption often falls short due to the complex dependencies that permeate sequences. Consequently, the performance of these methods often leaves much to be desired in real-world applications. For instance, as illustrated in Fig. \ref{fig: Illustrated example for mobile network}, non-i.i.d. alarm sequences in an operation and maintenance scenario highlight this issue. These event sequences are produced by network elements interconnected by a topological network, indicating the event sequences of different elements are dependent on each other, thereby violating the i.i.d. assumption. In this illustrated scenario, the event type $v_2$ in $n_2$ is only caused by the event types in topologically connected elements (specifically, $v_1$ in $n_2$). Existing methods would incorrectly identify the edge $v_3 \rightarrow v_2$ because they erroneously treat event sequences on different nodes as independent. By considering the prior topological network, we can correctly learn that $v_2$ is actually caused by $v_1$ and discern the true causal structure. Therefore, it is vital to exploit the topological network to model the dependencies among non-i.i.d. event sequences and accurately recover the causal structure from them.

This gives rise to two ensuing challenges: 1) how to explicitly model the generation process of topological event sequences, incorporating both the prior topological network and the latent causal structure, and 2) how to learn the latent Granger causal structure from topological event sequences. To tackle these challenges, Cai et al. \cite{cai2022thps} propose the topological Hawkes process (THP), which extends the time-domain convolution property of the Hawkes process to the time-graph domain. Nevertheless, THP's applicability in real-world scenarios remains limited due to the complex generation process and distribution of real-world scenarios, which are often too intricate to be accurately represented by the Hawkes process. 
Hence, it is crucial to devise a distribution-free generation model for non-i.i.d. event sequences that can incorporate both the topological network and the latent Granger causal structure. 

In response to these challenges, we develop a unified Topological Neural Poisson Auto-Regressive (TNPAR) model, comprising a generation process and an inference process. In the generation process, we introduce a variant of the Poisson Auto-Regressive model to depict the generation process of topological event sequences. By extending the Auto-Regressive model with Poisson Process and incorporating both the topological network and the Granger causal structure, TNPAR effectively overcomes the non-i.i.d. problem. By representing the distribution using a series of Poisson processes with varying parameters, TNPAR displays enhanced flexibility compared to existing methods that employ a single fixed distribution for the entire sequence. In the inference process, we design an amortized inference algorithm \cite{zhang2018advances,gershman2014amortized} to effectively learn the Granger causal structure. By integrating these generation and inference processes, we establish a unified likelihood function for TNPAR model, which can be optimized within an end-to-end paradigm.

In summary, our main contributions are as follows:
\begin{enumerate}
    \item We develop a topological neural Poisson auto-regressive method that accurately learns causal structure from topological event sequences.
    \item By incorporating the prior topological network into the generation model, we offer a distribution-free solution to the non-i.i.d. challenge in causal discovery.
    \item By treating the causal structure as a latent variable, we devise an amortized inference method to effectively identify the causal structure among the event types.
\end{enumerate}

\section{Related work}

\subsection{Temporal Point Process}
Temporal point processes are stochastic processes utilized to model event sequences. They can be bifurcated into two main categories: statistical point processes and neural point processes. On the one hand, statistical point processes emphasize designing appropriate intensity functions, the parameters of which often carry specific physical interpretations. Notable examples of statistical point processes encompass the Poisson process \cite{cox1955some}, Hawkes process \cite{hawkes1971spectra}, reactive point process \cite{ertekin2015reactive}, and self-correcting process \cite{isham1979self}, among others. On the other hand, neural point processes \cite{Shchur2021Neural} exploit the potent learning capabilities of neural networks to implement the intensity functions, often outperforming statistical point process methods in prediction performance. 

\subsection{Granger Causality for Event Sequences}
There are many types of methods for Granger causal discovery from event sequences. For instance, the Hawkes-process-based methods assume that the past events stimulate the occurrence of related events in the future if and only if the former Granger caused the latter. Typical Hawkes process-based methods \cite{zhou2013learning,xu2016learning} focus on designing appropriate intensity functions and regularized techniques. Recently, the THP algorithm \cite{cai2022thps} extends the Hawkes process to address the non-i.i.d. issue.
Other methods like Graphical Event Models (GEMs) \cite{bhattacharjya2018proximal} formalize Granger causality using process independence. In contrast, neural point process-based methods \cite{xiao2019learning,zhang2020cause} do not depend on specific assumptions of generation functions like the Hawkes process or GEMs. Instead, they employ neural networks to model event sequences with Granger causality. In another vein, our work is also related to the Granger causal discovery from time series. For instance, Brillinger \cite{brillinger1994time} aggregates event sequences into time series, which enables the analysis of event sequences using auto-regressive models. The Granger Causality alignment (GCA) model \cite{li2021transferable} combines the auto-regressive model with the generation model to learn the Granger causality. Amortized Causal Discovery (ACD) \cite{lowe2022amortized} trains an amortized model to infer causal graphs from time series. In addition, we provide a comparison of our work with related methods in Appendix A.

\section{Granger Causal Discovery Using Neural Poisson Auto-Regressive}

\subsection{Problem Definition}

Let an undirected graph $\mathcal{G}_N(\mathbf{N}, \mathbf{E}_N)$ represent the topological network among nodes $\mathbf{N}$, and a directed graph $\mathcal{G}_V(\mathbf{V}, \mathbf{E}_V)$ represent the Granger causal graph among event types $\mathbf{V}$. Here, $\mathbf{E}_N$ signifies the physically connected edges between nodes, and $\mathbf{E}_V$ denotes the causal edges between event types. In this context, an event could trigger its subsequent events in its own node and its topologically connected nodes. A series of \emph{topological event sequences} of length $m$, symbolized as $\mathbf{X}=\{(v_i, n_i, t_i) | i \in\{1, \ldots, m\}\}$, are generated by a Granger causal structure $\mathcal{G}_V(\mathbf{V}, \mathbf{E}_V)$ within a topological network $\mathcal{G}_N(\mathbf{N}, \mathbf{E}_N)$. Each item in the series consists of $v_i \in \mathbf{V}$ representing the event type, $n_i \in \mathbf{N}$ representing the topological node, and $t_i  \in [0,T]$ representing the occurrence timestamp. 

We can consider the topological event sequences $\mathbf{X}$ as a set of counting processes $\{\mathrm{C}_{v_i, n_j}(t)| v_i \in \mathbf{V}, n_j \in \mathbf{N}, t \in [0,T]\}$, where $\mathrm{C}_{v_i, n_j}(t)$ denotes the count of event type $v_i$ that has occurred up to $t$ at $n_j$. In this work, we divide the continuous interval $[0,T]$ into $\lceil T/\Delta \rceil$ small intervals, where $\Delta \in \mathbb{R}^+$ can be determined according to real-world applications. Then, the occurrence numbers of the counting processes can be denoted as $\mathbf{O}^{V,N}=\{O_t^{v_i, n_j} | t \in \{1,\ldots, \lceil T / \Delta \rceil \}, i \in \{1,\ldots, |\mathbf{V}| \}, j \in \{1,\ldots, |\mathbf{N}| \}\}$, where $O_t^{v_i, n_j}=\mathrm{C}_{v_i, n_j}(t \times \Delta) - \mathrm{C}_{v_i, n_j}((t-1) \times \Delta)$ represents the occurrence number of $v_i$ at $n_j$ within $((t-1) \times \Delta, t \times \Delta]$. Within each interval, we assume the counting process follows a Poisson process \cite{stoyan2013stochastic} with parameter $\lambda^{v_i, n_j}_t$. Thus, the probability of $O^{v_i, n_j}_t$ can be expressed as follows:
\begin{equation}
    \mathrm{P}(O^{v_i, n_j}_t \!=\!o) \!= \frac{(\lambda^{v_i, n_j}_t \Delta)^o}{o!} e^{-\lambda^{v_i, n_j}_t \Delta}; o=0, 1, ...
\label{eq: poisson process}
\end{equation}

Consequently, we can formulate the problem of causal discovery from topological event sequences as:

\begin{definition}[Granger causal discovery from topological event sequences]
\label{def: Granger causal discovery for event sequence generated discretely under topological network}
Given the occurrence numbers of topological event sequences, $\mathbf{O}^{V,N}=\{O_t^{v_i, n_j} | t \in \{1,\ldots, \lceil T / \Delta \rceil \}, i \in \{1,\ldots, |\mathbf{V}| \}, j \in \{1,\ldots, |\mathbf{N}| \}\}$, and the prior topological network $\mathcal{G}_N(\mathbf{N}, \mathbf{E}_N)$, the goal of Granger causal discovery from topological event sequences is to infer the Granger causal structure $\mathcal{G}_V(\mathbf{V}, \mathbf{E}_V)$ among the event types.
\end{definition}

In the following, we propose the TNPAR model to solve the above problem, comprising a generation process and an inference process.

\subsection{Generation Process of the Event Sequences via Topological Neural Poisson Auto-Regressive Model}
\begin{figure}
    \centering
    \includegraphics[width=0.35\textwidth]{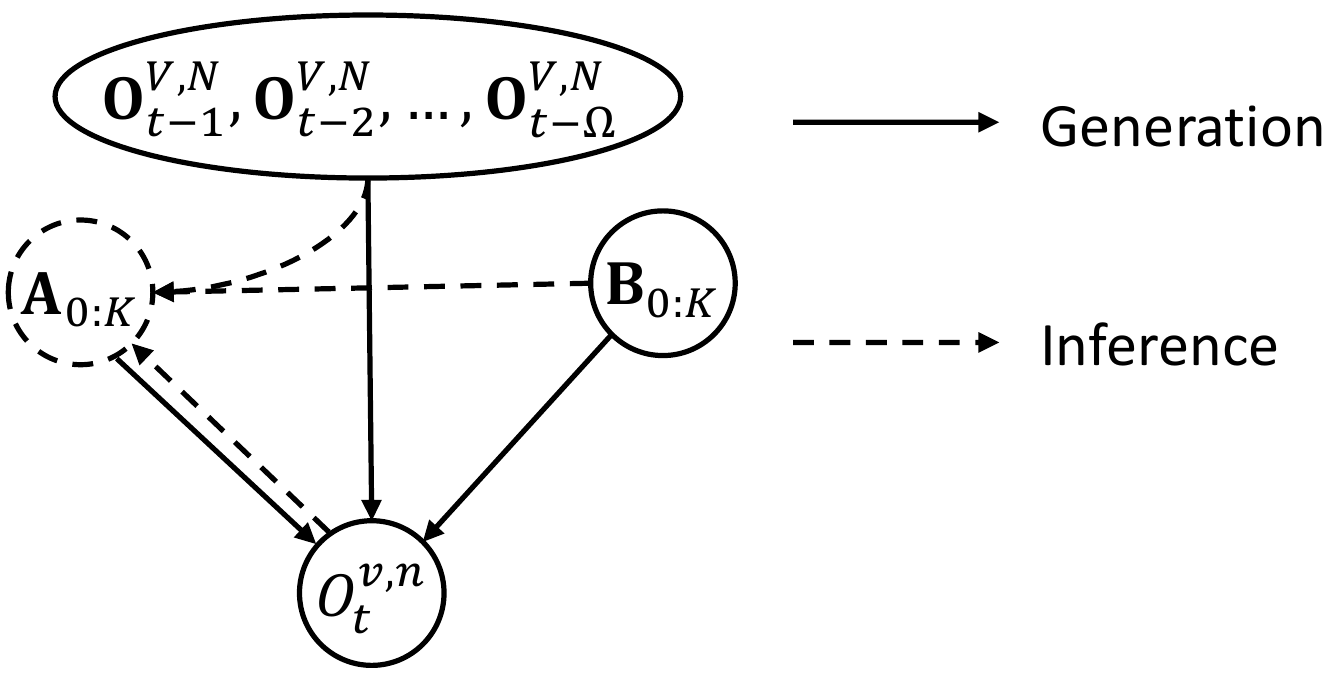}
    \caption{An illustration of the generation and inference processes for TNPAR. In this figure, $\mathbf{A}_{0:K}$ represents the causal matrices of $\mathcal{G}_V$, and $\mathbf{B}_{0:K}$ represents the topological matrices of $\mathcal{G}_N$. Solid lines signify the generation process for the data $O_{t}^{v,n}$, and dashed lines correspond to the inference process for the causal matrices $\mathbf{A}_{0:K}$. Observed variables are denoted by the solid circles, with the latent variables represented by the dashed circles.}
    \label{fig: generation example}
\end{figure}

TNPAR's generation process is depicted using solid lines in Fig. \ref{fig: generation example}. In this process, $O_t^{v_i,n_j}$ is determined by a combination of historical data $\{\mathbf{O}^{V,N}_{t-1}, \mathbf{O}^{V,N}_{t-2},\dots, \mathbf{O}^{V,N}_{t-\Omega}\}$, the causal matrices $\mathbf{A}_{0:K}$, and the topological matrices $\mathbf{B}_{0:K}$. Here, $\mathbf{A}_{0:K}$ represents a set of matrices $\{\mathbf{A}_0, \mathbf{A}_1, \dots, \mathbf{A}_K\}$, where $\mathbf{A}_k$ is a $|\mathbf{V}| \times |\mathbf{V}|$ binary matrix that indicates the Granger causality between event types at a geodesic distance \cite{bouttier2003geodesic} of $k$ within $\mathcal{G}_N$. Let $A_{k}^{i, j}$ denote the element in row $i$ and column $j$ of $\mathbf{A}_k$. Then, if $A_{k}^{i, j} = 0$, it implies that, at a geodesic distance of $k$, event type $v_i$ does not have a Granger causality with $v_j$. Otherwise, $A_k^{i,j}$ will be $1$. Similarly, $\mathbf{B}_{0:K}$ corresponds to another set of matrices $\{\mathbf{B}_0, \mathbf{B}_1,\dots, \mathbf{B}_K\}$, each of which $\mathbf{B}_k$ is a $|\mathbf{N}| \times |\mathbf{N}|$ binary matrix that indicates the physical connections of nodes at a geodesic distance of $k$. Let $B_{k}^{i, j}$ denote the element in row $i$ and column $j$ of $\mathbf{B}_k$. If the geodesic distance between node $n_i$ and $n_j$ is $k$ in $\mathcal{G}_N$, then $B_k^{i,j} = 1$. Otherwise, $B_k^{i,j}$ will be set to $0$.

In this work, we devise a model, combining the neural auto-regressive model and the Poisson process, to describe the above generation process. To initiate, we introduce a traditional neural auto-regressive model with $\Omega$ time lags:
\begin{equation}
    O_t^{v_i} = \mathrm{g}(\mathbf{O}^V_{t-1 : t-\Omega}) + \epsilon_t^{v_i},
    \label{eq:neural_AR}
\end{equation}
where $\mathbf{O}^V_{t-1 : t-\Omega} = \{\mathbf{O}^V_{t-1}, \mathbf{O}^V_{t-2},\dots, \mathbf{O}^V_{t-\Omega}\}$ represents a set of historical data, $\mathrm{g}(\cdot)$ is a nonlinear function implemented by a neural network that generates data in the $t$-th time interval based on the historical data, and $\epsilon_t^{v_i}$ denotes the noise term for event type $v_i$ in the $t$-th time interval. 
In order to incorporate the topological relationships of nodes into our model, we extend $\mathbf{O}^V_{t-1 : t-\Omega}$ to include the topologically connected historical data, using the function $\mathrm{f}^{\mathcal{G}_N}_{ n_j}(\mathbf{O}^{V,N}_{t-1 : t-\Omega}, \mathbf{B}_{0:K})$. This function applies a filter to set the non-topologically connected data for node $n_j$ to $0$ in $\mathbf{O}^{V,N}_{t-1 : t-\Omega}$ based on $\mathbf{B}_{0:K}$, which effectively removes the non-topologically connected information in the generation process of $O_t^{v_i}$. Building upon this, our auto-regressive model, incorporating topological information, can be expressed as follows:
\begin{equation}
    O_t^{v_i,n_j} = \mathrm{g}(\mathrm{f}^{\mathcal{G}_N}_{ n_j}(\mathbf{O}^{V,N}_{t-1 : t-\Omega}, \mathbf{B}_{0:K})) + \epsilon_t^{v_i, n_j}.
    \label{eq:NAR}
\end{equation}

Furthermore, to incorporate the Granger causal structure into the above model, we substitute $\mathbf{O}^{V,N}_{t-1 : t-\Omega}$ in Eq. (\ref{eq:NAR}) with a function $\mathrm{f}^{\mathcal{G}_V}_{v_i}(\mathbf{O}^{V,N}_{t-1 : t-\Omega}, \mathbf{A}_{0:K}\!)$. This function acts as a filter to set the historical data that do not have a Granger causality with $v_i$ to $0$ in $\mathbf{O}^{V,N}_{t-1 : t-\Omega}$ based on $\mathbf{A}_{0:K}$. By incorporating both the topological network and the Granger causal structure, we propose the following model for generating topological event sequences:
\begin{equation}
\label{eq: TNPAR}
O_t^{v_i,n_j} \!\!=\!\mathrm{g}(\mathrm{f}^{\mathcal{G}_N}_{n_j}\!(\mathrm{f}^{\mathcal{G}_V}_{v_i}(\mathbf{O}^{V,N}_{t-1 : t-\Omega}, \mathbf{A}_{0:K}\!), \mathbf{B}_{0:K})) \!+ \epsilon_t^{v_i, n_j}.
\end{equation}

\subsection{Inference Process of the Granger Causal Structure via Amortized Inference}
The inference process of TNPAR is depicted using dashed lines in Fig. \ref{fig: generation example}. As shown in the diagram, the causal matrices $\mathbf{A}_{0:K}$ are inferred based on a combination of the current occurrence number $O_t^{v_i,n_j}$, historical occurrence numbers $\mathbf{O}^{V,N}_{t-1 : t-\Omega}$ and the topological network $\mathbf{B}_{0:K}$. Note that the causal relationship is invariant across the varying samples. This nature inspires us to implement the inference process with amortized inference \cite{zhang2018advances,gershman2014amortized}, in which a shared function is employed to predict the varying posterior distribution of each sample. Consequently, we consider the generation process function, defined in Eq. (\ref{eq: TNPAR}), as the shared function and designate the causal matrices as the function's parameter to implement the amortized inference process.


In more detail, to amortize the inference of the shared function, we optimize an inference model to predict the distribution of causal matrices. Given the data $O_t^{v_i,n_j}$ and $\mathbf{O}^{V,N}_{t-1 : t-\Omega}$, as well as the topological network $\mathbf{B}_{0:K}$, we use $\mathrm{q}_{\phi}(\mathbf{A}_{0:K}|O_t^{v_i,n_j},\mathbf{O}^{V,N}_{t-1 : t-\Omega},\mathbf{B}_{0:K})$ to approximate the true distribution $\mathrm{p}_{\theta}(\mathbf{A}_{0:K}|O_t^{v_i,n_j},\mathbf{O}^{V,N}_{t-1 : t-\Omega},\mathbf{B}_{0:K})$. Here, $\phi$ and $\theta$ represent the variational and generation parameters, respectively. The value of $\theta$ and $\phi$ can be learned by maximizing the ensuing likelihood function for each $O_t^{v_i,n_j}$:
\begin{equation}
\begin{aligned}
    & \log \mathrm{P}(O_t^{v_i,n_j}|\mathbf{O}^{V,N}_{t-1 : t-\Omega}, \mathbf{B}_{0:K}) & \\
    = & \mathrm{D}_{KL}(\mathrm{q}^A_{\phi}||\mathrm{p}^A_{\theta}) \!+\! \mathcal{L}(\theta,\phi;O_t^{v_i,n_j}|\mathbf{O}^{V,N}_{t-1 : t-\Omega}, \mathbf{B}_{0:K}),
\end{aligned}
\label{eq: logarithm of joint likelihood}
\end{equation}
where $\mathrm{q}^A_{\phi} = \mathrm{q}_{\phi}(\mathbf{A}_{0:K}|O_t^{v_i,n_j},\mathbf{O}^{V,N}_{t-1 : t-\Omega},\mathbf{B}_{0:K})$, $\mathrm{p}^A_{\theta} = \mathrm{p}_{\theta}(\mathbf{A}_{0:K}|O_t^{v_i,n_j},\mathbf{O}^{V,N}_{t-1 : t-\Omega},\mathbf{B}_{0:K})$. In Eq. (\ref{eq: logarithm of joint likelihood}), the right-hand side's (RHS) first term represents the Kullback-Leibler (KL) divergence of the approximate posterior from the true posterior, and the RHS's second term is the variational evidence lower bound (ELBO) on the log-likelihood of $O_t^{v_i,n_j}$. By defining the ELBO as $\mathcal{L}_{e}$ and considering the KL divergence is non-negative, we can deduce from the Eq. (\ref{eq: logarithm of joint likelihood}) that:\begin{equation}
\begin{aligned}
    & \log \mathrm{P}(O_t^{v_i,n_j}|\mathbf{O}^{V,N}_{t-1 : t-\Omega}, \mathbf{B}_{0:K}) \geq \mathcal{L}_{e} \\
    = & \mathbb{E}_{\mathrm{q}^A_{\phi}} [-\log \mathrm{q}^A_{\phi} + \log \mathrm{p}_{\theta}(O_t^{v_i,n_j}, \mathbf{A}_{0:K}|\mathbf{O}^{V,N}_{t-1 : t-\Omega},\mathbf{B}_{0:K})].
\end{aligned}
\end{equation}

The ELBO can be further expressed as: 
\begin{equation}
\begin{aligned}
\mathcal{L}_{e} = & \mathbb{E}_{\mathrm{q}^A_{\phi}} [\log \mathrm{p}_{\theta}(O_t^{v_i,n_j}| \mathbf{A}_{0:K}, \mathbf{O}^{V,N}_{t-1 : t-\Omega},\mathbf{B}_{0:K})]
 \\ & - \mathrm{D}_{KL}(\mathrm{q}^A_{\phi}||\mathrm{p}_{\theta}(\mathbf{A}_{0:K})).
\end{aligned}
\label{eq: elbo}
\end{equation}

Building upon this equation, the goal of the inference task is to optimize $\mathcal{L}_{e}$ with respect to both $\phi$ and $\theta$, thereby obtaining the posterior distribution of latent Granger causal structure given the data and the topological network.

\subsection{Theoretical Analysis}
In this subsection, we will theoretically analyze the identifiability of the proposed model. To do so, according to the definition in \cite{tank2018neural}, we begin with the definition of the Granger non-causality of multi-type event sequences as follows:

\begin{definition}[Granger non-causality of multi-type event sequences] 
Given data $\mathbf{O}^{V}$ during time interval $t \in \{1,\dots, \lceil T/\Delta \rceil \}$, which is generated according to Eq. (\ref{eq:neural_AR}). For the original sample values of historical variables $\mathbf{O}^{v_1}_{t-1 : t-\Omega}, ..., \mathbf{O}^{v_{|V|}}_{t-1 : t-\Omega}$, let $\mathbf{\hat{O}}^{v_i}_{t-1 : t-\Omega}$ represent the different sample values of the historical variables $\mathbf{O}^{v_i}_{t-1 : t-\Omega}$, (i.e., $\mathbf{O}^{v_i}_{t-1 : t-\Omega} \neq \mathbf{\hat{O}}^{v_i}_{t-1 : t-\Omega}$). We can determine the Granger non-causality of event type $v_i$ with respect to event type $v_j$, if the following condition holds for all $t$:
\begin{equation*}
\begin{aligned}
&\mathrm{g}(\mathbf{O}^{v_1}_{t-1 : t-\Omega}, ..., \mathbf{O}^{v_i}_{t-1 : t-\Omega} ,..., \mathbf{O}^{v_{|V|}}_{t-1 : t-\Omega}) \\
= &\mathrm{g}(\mathbf{O}^{v_1}_{t-1 : t-\Omega}, ..., \mathbf{\hat{O}}^{v_i}_{t-1 : t-\Omega} ,..., \mathbf{O}^{v_{|V|}}_{t-1 : t-\Omega}).  \\
\end{aligned}
\end{equation*}

That is, $O_t^{v_j}$ is invariant to $\mathbf{O}^{v_i}_{t-1 : t-\Omega}$ with $\mathrm{g}(\cdot)$.
\end{definition}

By extending the above definition to the topological event sequences, we define the notion of Granger non-causality of topological event sequences as follows:

\begin{definition}[Granger non-causality of topological event sequences]
Assume that the occurrence numbers of event sequences are generated by Eq. (\ref{eq: TNPAR}). For the original sample values of historical variables  $\mathbf{O}^{v_1, n_1}_{t-1 : t-\Omega}, ..., \mathbf{O}^{v_{|V|}, n_{|N|}}_{t-1 : t-\Omega}$, let $\mathbf{\hat{O}}^{v_i, n_k}_{t-1 : t-\Omega}$ represent the different sample values of the historical variables $\mathbf{O}^{v_i, n_k}_{t-1 : t-\Omega}$, (i.e., $\mathbf{\hat{O}}^{v_i, n_k}_{t-1 : t-\Omega} \neq \mathbf{O}^{v_i, n_k}_{t-1 : t-\Omega}$). We can determine the Granger non-causality of event type $v_i$ with respect to event type $v_j$ under topological network $\mathcal{G}_N$, if the following condition holds for all $t$:
\begin{equation*}
\begin{aligned}
& \mathrm{g}(\mathbf{O}^{v_1, n_1}_{t-1 : t-\Omega}, ..., \mathbf{O}^{v_i, n_1}_{t-1 : t-\Omega} ,..., \mathbf{O}^{v_i, n_{|N|}}_{t-1 : t-\Omega} ,..., \mathbf{O}^{v_{|V|}, n_{|N|}}_{t-1 : t-\Omega})\\
= & \mathrm{g}(\mathbf{O}^{v_1, n_1}_{t-1 : t-\Omega}, ..., \mathbf{\hat{O}}^{v_i, n_1}_{t-1 : t-\Omega} ,..., \mathbf{\hat{O}}^{v_i, n_{|N|}}_{t-1 : t-\Omega} ,..., \mathbf{O}^{v_{|V|}, n_{|N|}}_{t-1 : t-\Omega}).
\end{aligned}
\end{equation*}

That is, $O_t^{v_j, n_k}$ is invariant to $\mathbf{O}^{v_i, n_1}_{t-1 : t-\Omega} ,..., \mathbf{O}^{v_i, n_{|N|}}_{t-1 : t-\Omega}$ with $\mathrm{g}(\cdot)$ across all nodes in topological network $\mathcal{G}_N$.
\label{definition: topological granger causality}
\end{definition}

Definition \ref{definition: topological granger causality} implies that if two event types are Granger non-causality among a topological network, then they are Granger non-causality across all nodes. Based on this definition, we can derive a proposition about the identification of the model as follows.

\begin{proposition}
Given the Granger causal structure $\mathcal{G}_V(\mathbf{V}, \mathbf{E}_V)$, the topological network $\mathcal{G}_N(\mathbf{N}, \mathbf{E}_N)$ and the max geodesic distance $K$, along with the assumption that the data generation process adheres to Eq. (\ref{eq: TNPAR}), we can deduce that $v_i \rightarrow v_j \notin \mathbf{E}_V$, if and only if $A_{k}^{v_i, v_j} = 0$ for every  $k \in \{0,...,K\}$.
\label{proposition: TNPAR}
\end{proposition}

The proof of Proposition \ref{proposition: TNPAR} is provided in Appendix B. This proposition inspires a methodology for TNPAR to identify Granger causality, which involves evaluating whether the elements of  $\mathbf{A}_{0:K}$ are zero.

\section{Algorithm and Implementation}

\subsection{Inference Process of TNPAR} 
In this subsection, we employ an encoder to implement the above inference process of TNPAR. Here, the $O_t^{v_i,n_j}$ and $\mathbf{O}^{V,N}_{t-1 : t-\Omega}$ are used jointly as inputs to the encoder, with $\mathbf{B}_{0:K}$ serving as a mask. The output is the variational posterior of the Granger causal structure, represented as $\mathrm{q}_{\phi}(\mathbf{A}_{0:K}|O_t^{v_i,n_j},\mathbf{O}^{V,N}_{t-1 : t-\Omega},\mathbf{B}_{0:K})$.

To elaborate, each element within $\mathbf{O}^{V,N}_{t-1 : t-\Omega}$ is initially filtered in accordance with both $\mathbf{B}_{0:K}$ and $n_j$. The resulting values are then combined with $O_t^{v_i,n_j}$ and converted into a vector, where each item denotes the representation of the occurrence numbers at a geodesic distance of $k$ for $n_j$. Following this, the encoder applies a multilayer perception (MLP) \cite{hastie2009elements} to the input, which facilitates the propagation of input information across multiple layers and nonlinear activation functions. Ultimately, the encoder generates the posterior distribution of Granger causal structure  $\mathrm{q}_{\phi}(\mathbf{A}_{0:K}|O_t^{v_i,n_j},\mathbf{O}^{V,N}_{t-1 : t-\Omega},\mathbf{B}_{0:K})$. 

In a manner similar to variational inference for categorical latent variables, we assume that the posterior distribution of each causal edge in the Granger causal structure follows a Bernoulli distribution. More specifically, the encoder produces a latent vector $\mathbf{Z} \in [0, 1]^{|\mathbf{V}| \times |\mathbf{V}| \times K}$, where $z_{i, j, k}$, the $(i \times j \times k)$-th element of $\mathbf{Z}$, represents the parameter of Bernoulli distribution for corresponding edge. Consequently, $\mathrm{q}_{\phi}(\mathbf{A}_{0:K}|O_t^{v_i,n_j},\mathbf{O}^{V,N}_{t-1 : t-\Omega},\mathbf{B}_{0:K})$ can be expressed as:\begin{equation}
\small
\begin{aligned}
\mathrm{q}_{\phi}(\mathbf{A}_{0:K}|O_t^{v_i,n_j},\mathbf{O}^{V,N}_{t-1 : t-\Omega},\mathbf{B}_{0:K}) \\
= \prod_{i=1}^{|\mathbf{V}|} \prod_{j=1}^{|\mathbf{V}|} \prod_{k=0}^{K} \mathrm{q}_{\phi} (A_{k}^{v_i, v_j} | z_{i, j, k}) \\
\textrm{  with  } \: \: \: \mathrm{q}_{\phi}(A_{k}^{v_i, v_j} = 1 | z_{i, j, k}) = \sigma_{\beta} (z_{i, j, k}). 
\end{aligned}
\label{eq: encoder}
\end{equation}

Here, $\sigma_{\beta} (x) = 1 / (1 + \exp (-\beta x))$ denotes the sigmoid function with an inverse temperature parameter $\beta$.

Furthermore, for backpropagation through the discrete distribution $\mathrm{q}_{\phi}(\mathbf{A}_{0:K}|O_t^{v_i,n_j},\mathbf{O}^{V,N}_{t-1 : t-\Omega},\mathbf{B}_{0:K})$, we employ the Gumbel-Softmax trick \cite{jang2017categorical} to generate a differentiable sample as: \begin{equation}
\small
\hat{A}_{k}^{v_i, v_j} \!=\! \frac{\exp(\log(\mathrm{q}_{\phi}(A_{k}^{v_i, v_j} \!=\! 1 | z_{i, j, k}) + \varepsilon_1) / \tau)}{\sum_{e=0}^{1} \exp(\log(\mathrm{q}_{\phi}(A_{k}^{v_i, v_j} \!=\! e | z_{i, j, k}) + \varepsilon_e) / \tau)},
\end{equation}
where each $\varepsilon_{*}$ is an i.i.d. sample drawn from a standard Gumbel distribution, while $\tau$ is the temperature parameter controlling the randomness of $\hat{A}_{k}^{v_i, v_j}$.

\subsection{Generation Process of TNPAR} 
Next, we propose a decoder to implement the generation process of $\mathbf{O}_t^{V,N}$. In this part, $\mathbf{O}^{V,N}_{t-1 : t-\Omega}$ is filtered according to both $\hat{\mathbf{A}}_{0:K}$ and $\mathbf{B}_{0:K}$, as specified in Eq. (\ref{eq: TNPAR}). This filtered result is then used as input to the decoder. Notably, for neural network training, we treat $\hat{\mathbf{A}}_{0:K}$ as a weight matrix during training, and as a mask during testing. This can be accomplished using soft/hard Gumbel-Softmax sampling. Similar to the encoder, an MLP is used in the decoder.

Aligning with numerous works on counting processes that assume the distribution of occurrence numbers $O_t^{v_i,n_j}$ in each time interval is stationary \cite{stoyan2013stochastic,babu1996spatial}, we model each $\mathrm{p}_{\theta}(O_t^{v_i,n_j}|\mathbf{O}^{V,N}_{t-1 : t-\Omega}, \mathbf{A}_{0:K}, \mathbf{B}_{0:K})$ using a Poisson process as denoted in Eq. (\ref{eq: poisson process}). In this context, the conditional distribution of $O_t^{v_i,n_j}$ can be formalized as:\begin{equation}
\small
\begin{aligned}
    \mathrm{p}_{\theta}(O_t^{v_i,n_j}|\mathbf{O}^{V,N}_{t-1 : t-\Omega}, \mathbf{A}_{0:K}, \mathbf{B}_{0:K}) \\
    \!= \frac{(\lambda_t^{v_i,n_j}\! \Delta)^{O_t^{v_i,n_j}}}{O_t^{v_i,n_j}!} e^{-\lambda_t^{v_i,n_j}\! \Delta}, 
\end{aligned}
\label{eq: decoder}
\end{equation}
where $\lambda_t^{v_i,n_j}$, provided by the output of the decoder, represents the intensity parameter of the Poisson process.

\subsection{Optimization of TNPAR}
Based on the above analysis, we adopt the ELBO $\mathcal{L}_{e}$, defined in Eq. (\ref{eq: elbo}), as the objective function to estimate $\phi$ and $\theta$ and infer the Granger causal structure. To adapt to real-world scenarios, we further introduce acyclic constraint and sparsity constraint into the objective function.

In many real-world applications, the causal structure $\mathcal{G}_V$ is acyclic. For instance, based on experts' prior knowledge, the causal structure in the aforementioned operation and maintenance scenario of a mobile network is a DAG (Directed Acyclic Graph). With this in mind, we introduce an acyclic constraint proposed by \cite{yu2019dag}, denoted as:\begin{equation}
\small
\mathrm{h}(\mathbf{G}) \equiv \mathrm{tr}((\mathbf{I} + \frac{1}{|\mathbf{V}|} \mathbf{G})^{|\mathbf{V}|}) - |\mathbf{V}| = 0,
\label{eq: acyclic}
\end{equation}
where $\mathbf{G}$ is a matrix in which $G_{i, j} = 0$ signifies that $v_i$ has no Granger causality with $v_j$. Conversely, $G_{i, j} > 0$. The function $\mathrm{tr}(\cdot)$ calculates the trace of an input matrix. Then, a causal structure $\mathbf{G}$ is acyclic if and only if $\mathrm{h}(\mathbf{G}) = 0$; Otherwise, $\mathrm{h}(\mathbf{G}) > 0$. In this study, we set $G_{i, j} = \sum_{k=0}^{K} A_{k}^{v_i, v_j}$ and populate the diagonal of $\mathbf{G}$ with $0$, which implies that self-excitation is permitted in an acyclic graph. Subsequently, we design an acyclic regulared term $\mathcal{L}_{c}$ as:\begin{equation}
\small
\mathcal{L}_{c} = \mathbb{E}_{\mathrm{q}_{\phi}} [\mathrm{tr}((\mathbf{I} + \frac{1}{|\mathbf{V}|} \mathbf{G})^{|\mathbf{V}|}) - |\mathbf{V}|] = 0.
\label{eq: dag}
\end{equation}

Because the causal structure tends to be sparse in most real-world applications, we employ a $\ell_1$-norm sparsity constraint as follows: \begin{equation}
\small
    \mathcal{L}_{s} = \mathbb{E}_{\mathrm{q}_{\phi}} [\sum_{i=1}^{|\mathbf{V}|} \sum_{j=1}^{|\mathbf{V}|} \sum_{k=0}^{K} A_{k}^{v_i, v_j}] \leq \kappa,
\label{eq: sparse}
\end{equation}
where $\kappa$ is a small positive constant. Therefore, the training procedure boils down to the following optimization:
\begin{equation*}
\small
\max ~ \mathcal{L}_{e} \quad \text{s.t.} ~ \mathcal{L}_c=0, ~ \mathcal{L}_s\leq \kappa.
\end{equation*}

By leveraging the Lagrangian multiplier method, we define the total loss function $\mathcal{L}_{total}$ as:\begin{equation}
\small
\mathcal{L}_{total} = - \mathcal{L}_{e} + \lambda_{c} \mathcal{L}_{c} + \lambda_{s} \mathcal{L}_{s},
\label{eq: total loss}
\end{equation}
where $\lambda_{c}$ and $\lambda_s$ denote the regularized hyperparameters.
 
In summary, our proposed model is trained using the following objective:\begin{equation}
\small
    (\phi^\star, \theta^\star) = \mathop{\arg\min}\limits_{\phi, \theta} \mathcal{L}_{total}.
\end{equation}

For model training, we can adopt Adam \cite{kingma2014adam} stochastic optimization.

\section{Experiment}
In this section, we apply the proposed TNPAR method and the baselines to both simulated data and metropolitan cellular network alarm data. For all methods, we employ five different random seeds and present the results in graphical format, with error bars included. Our evaluation metrics for the experiments include Precision, Recall, F1 score \cite{powers2011evaluation}, Structural Hamming Distance (SHD), and Structural Intervention Distance (SID) \cite{peters2015structural}. Specifically, Precision refers to the fraction of predicted edges that exist among the true edges. Recall is the fraction of true edges that have been successfully predicted. F1 score is the weighted harmonic mean of both Precision and Recall, and is calculated as $F1 = \frac{2 \times Precision \times Recall}{Precision + Recall}$. SHD represents the number of edge insertions, deletions, or flips needed to transform a graph into another graph. SID is a measure that quantifies the closeness between two DAGs based on their corresponding causal inference statements.

\subsection{Dataset}
\subsubsection{Simulated Data}
The simulated data is generated as follows: a) a Granger causal structure $\mathcal{G}_V$ and a topological network $\mathcal{G}_N$ are randomly generated; b) root event records are generated via the Poisson process using a base intensity parameter $\mu$ in the Hawkes process. These root event records are spontaneously generated by the system; c) based on the root event records, propagated event records are discretely generated according to both the time interval $\Delta$ and the excitation intensity $\alpha$. Here, $\alpha$ represents the event excitation intensity in the Hawkes process. Given that the event sequences can be quite sparse in real-world data, we offer a time interval parameter $\Delta$ to the generation process, dividing the time domain $[0, T]$ to small intervals with indexes as $\{1, \ldots, \lceil T/\Delta \rceil \}$. Then, the event records can be summarized within the same timestamp. Note that the $\Delta \geq 0$. If $\Delta = 0$, it implies the use of the original event sequences.

All simulated data are generated by varying the parameters of the generation process one at a time while maintaining default parameters. We set these default parameters according to the data generation process used in the PCIC 2021 competition\footnote{\url{https://competition.huaweicloud.com/information/1000041487/dataset}}, inclusive of the additional parameter $\Delta$, as follows: 
$\# \textrm{nodes}=40$, $\# \textrm{event type}=20$, $\# \textrm{sample size}=20000$, $\mu=[0.00003, 0.00005]$, $\alpha=[0.02,0.03]$, $\Delta=2$.

\subsubsection{Metropolitan Cellular Network Alarm Data}
This real-world dataset is collected in a business scenario by a multinational communications company and is available in the PCIC 2021 competition. The data comprises a series of alarm records generated within a week according to both a topological network $\mathcal{G}_N$ and a causal structure $\mathcal{G}_V$. Specifically, the topological network $\mathcal{G}_N$ includes $3087$ network elements. The alarm records encompass $18$ alarm event types, totaling $228,030$ alarm event records. It is noteworthy that, due to the characteristics of the equipment and as confirmed by experts, the collected timestamps of alarm events exhibit certain time intervals.

\begin{figure*}[t]
	\centering
	\subfigure[Sensitivity to Range of $\alpha (\times 1e-2)$]{
		\includegraphics[width=0.32\textwidth]{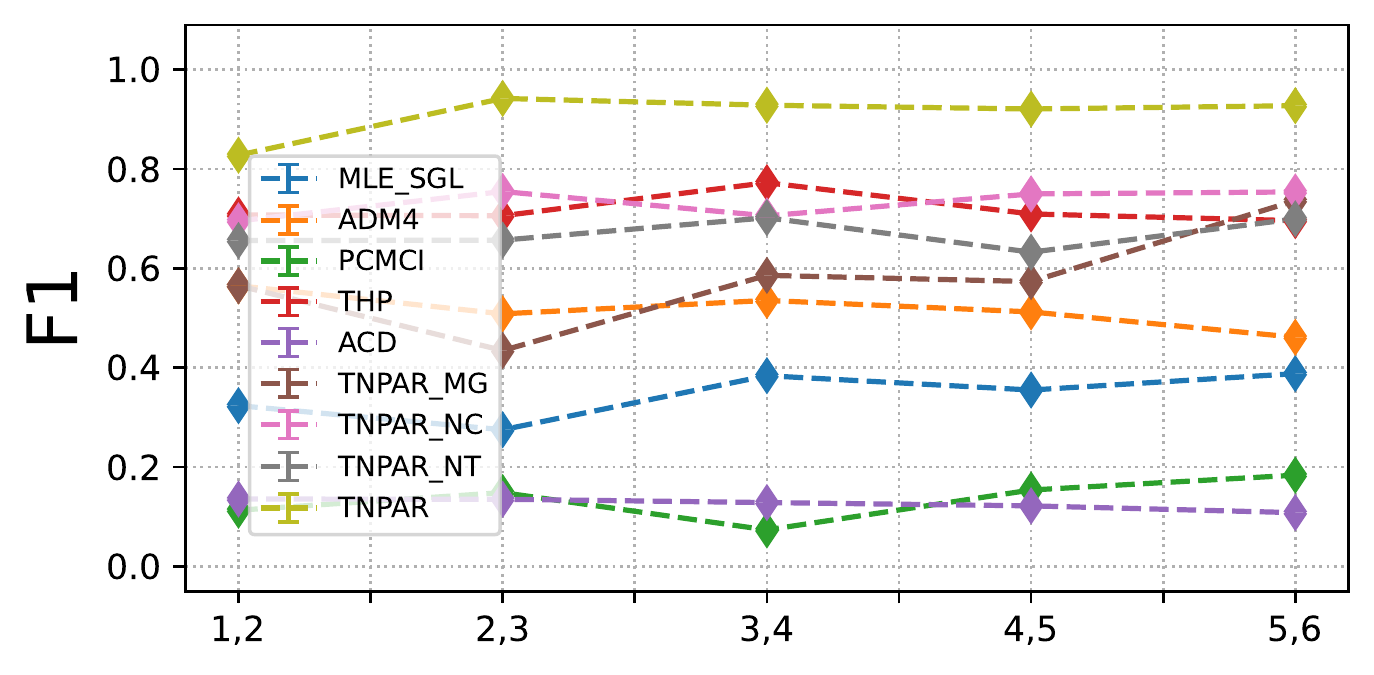}
		\label{f2:a}
	}
	\subfigure[Sensitivity to Range of $\mu (\times 1e-5)$]{
	\includegraphics[width=0.32\textwidth]{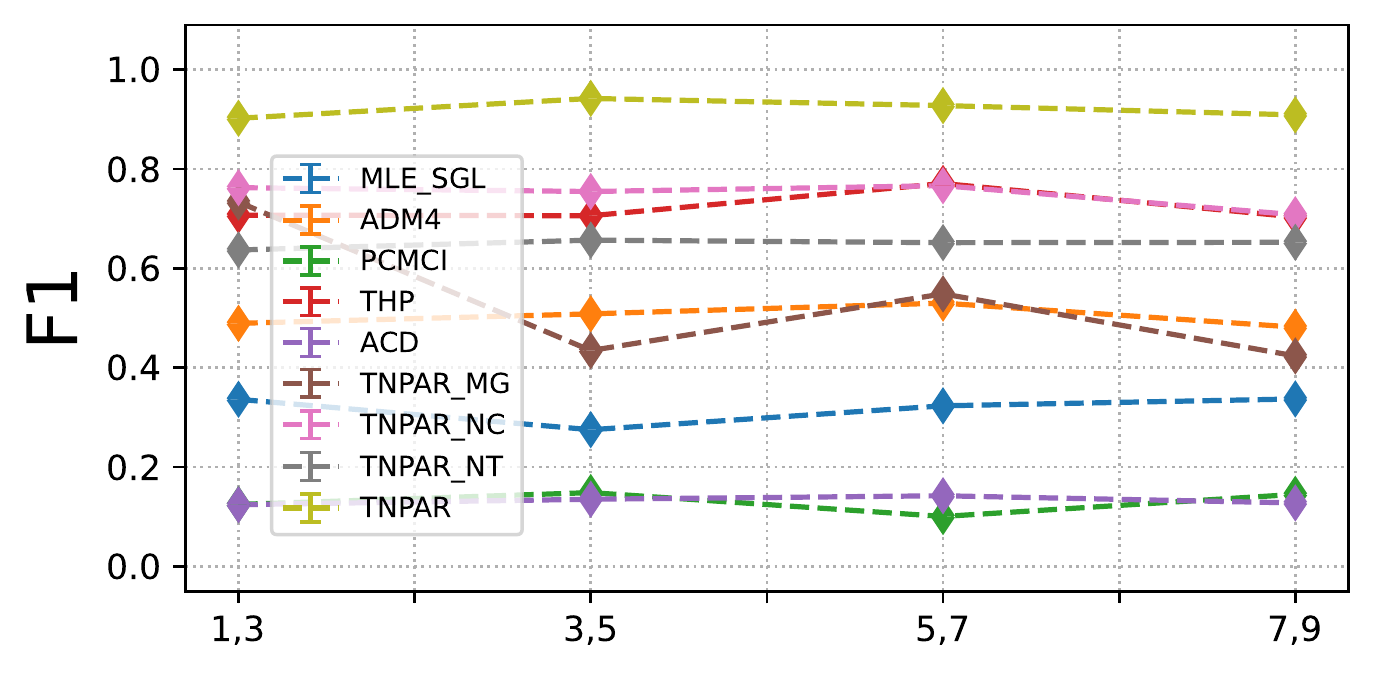}
	\label{f2:b}
}
	\subfigure[Sensitivity to Sample Size]{
	\includegraphics[width=0.32\textwidth]{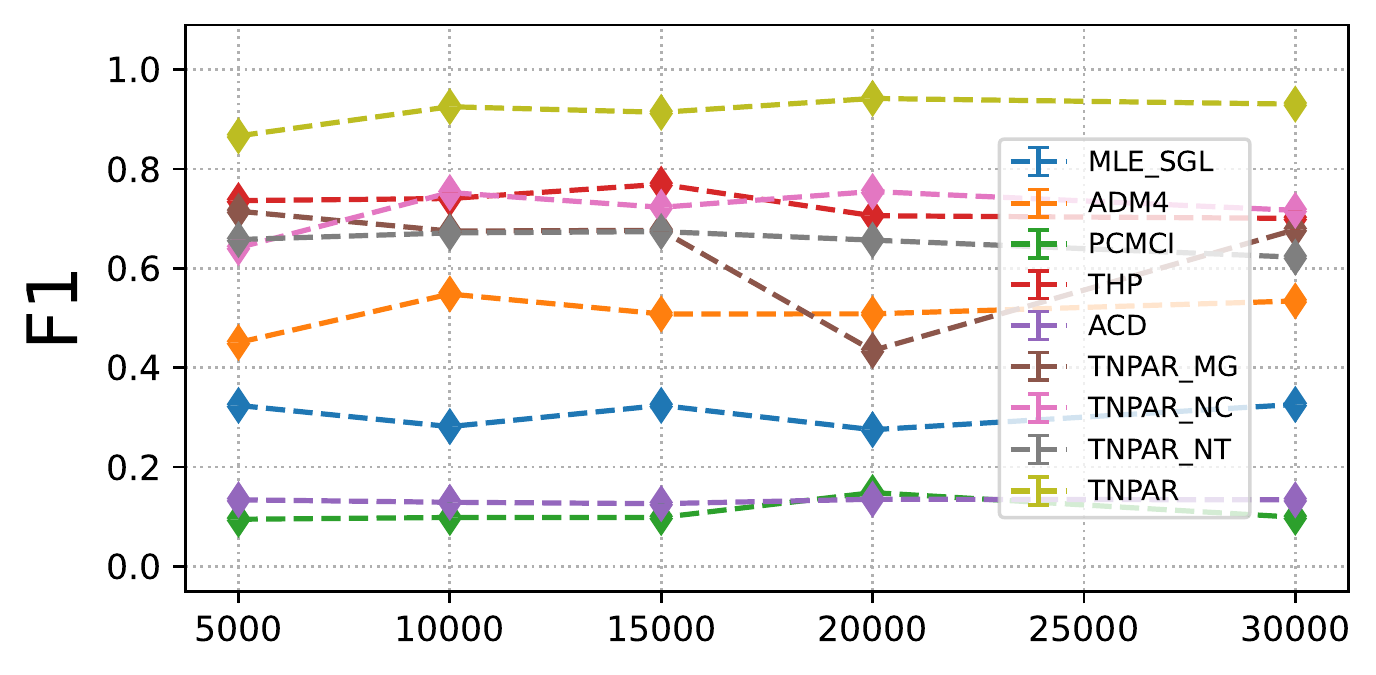}
	\label{f2:c}
}
	\subfigure[Sensitivity to $\Delta$]{
	\includegraphics[width=0.32\textwidth]{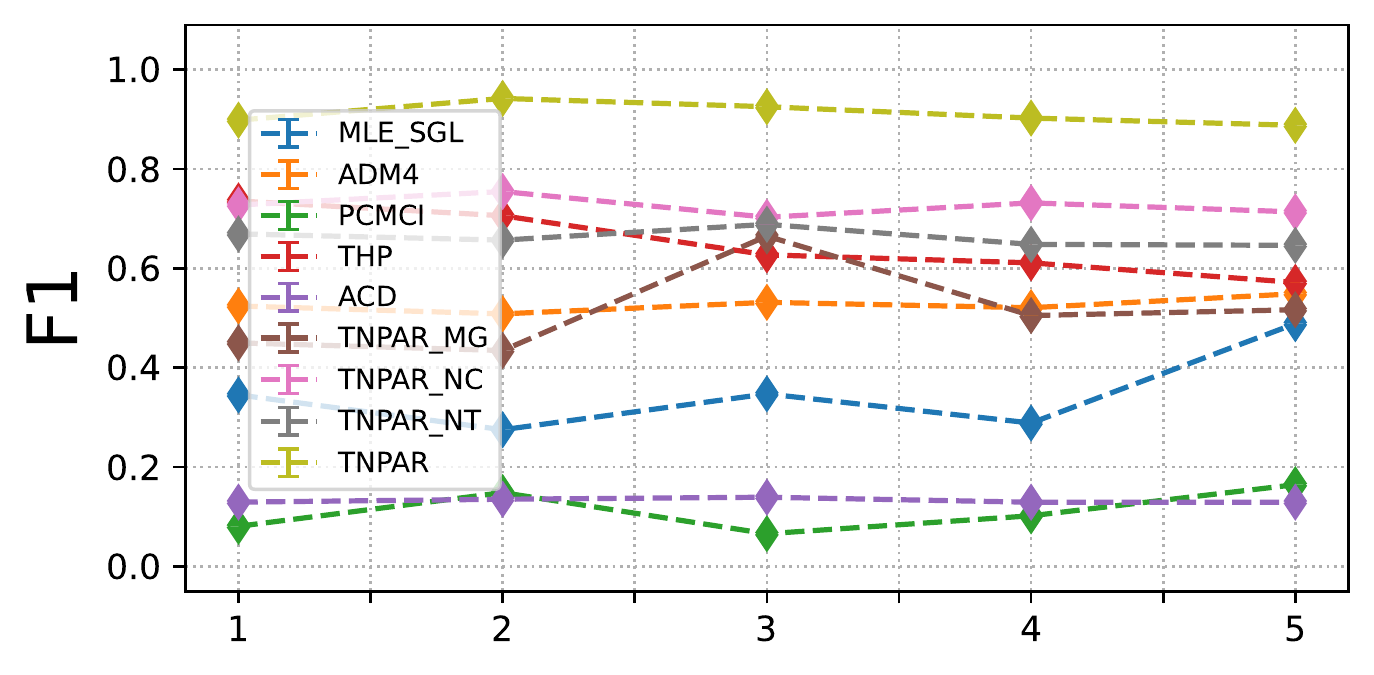}
	\label{f2:d}
}
	\subfigure[Sensitivity to Num. of Event Types]{
	\includegraphics[width=0.32\textwidth]{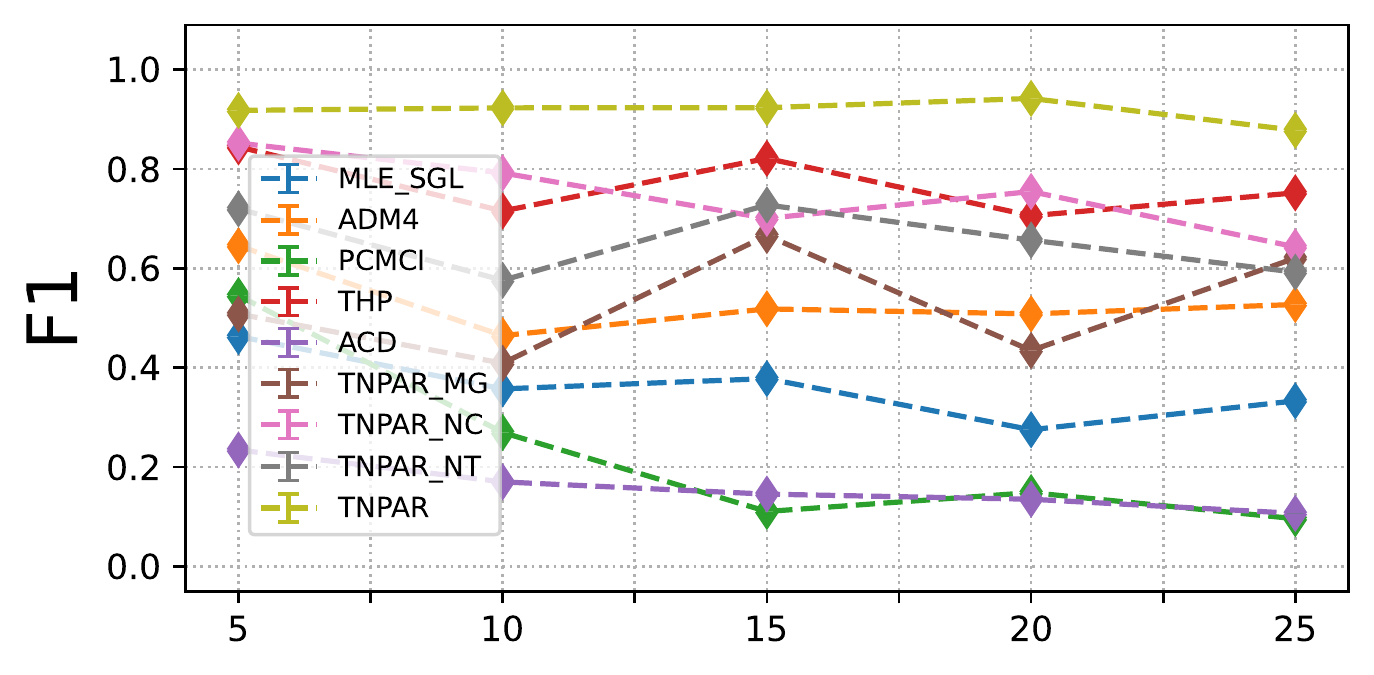}
	\label{f2:e}
}
	\subfigure[Sensitivity to Num. of Nodes]{
	\includegraphics[width=0.32\textwidth]{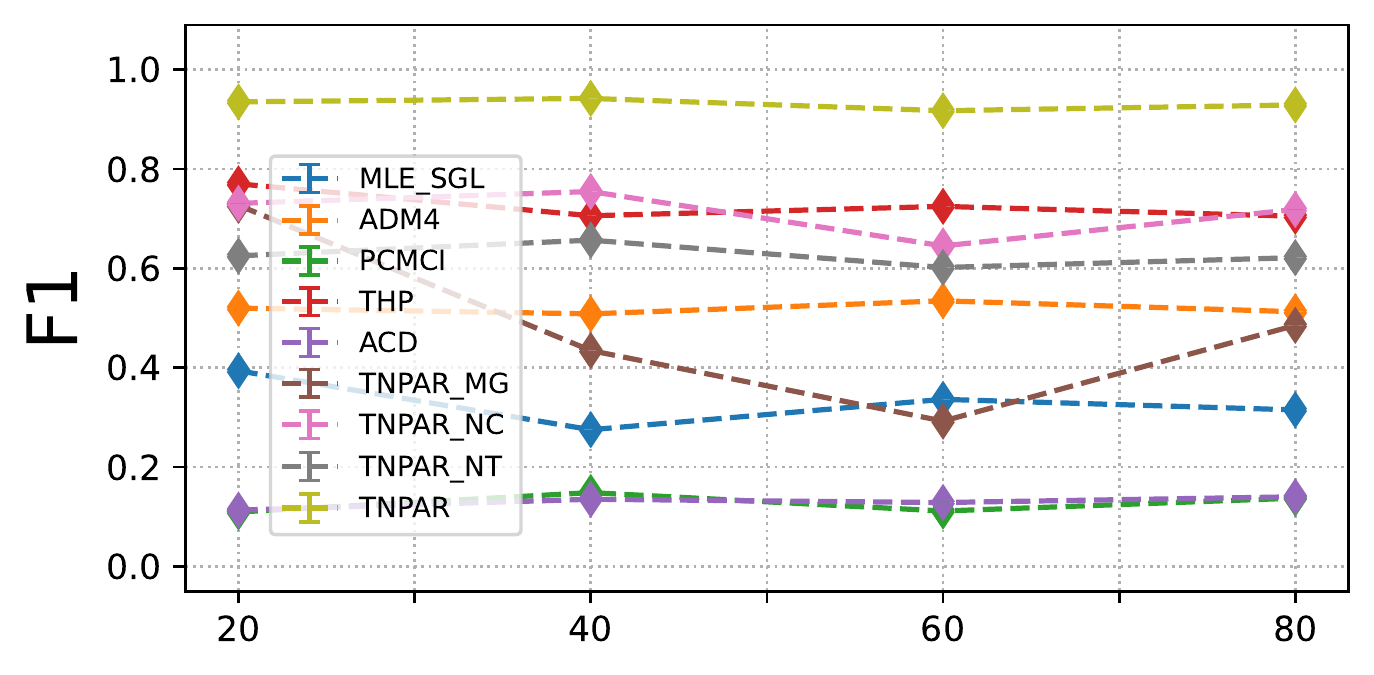}
	\label{f2:f}
}
	\caption{Results on the simulated data}	
	\label{f2}
\end{figure*}

\subsection{Baselines and Model Variants}
The following causal discovery methods are chosen as baseline comparisons in our experiments: PCMCI \cite{runge2020discovering}, ACD \cite{lowe2022amortized}, MLE\_SGL \cite{xu2016learning}, ADM4 \cite{zhou2013learning}, and THP \cite{cai2022thps}. In order to assess the components of our proposed method, we also introduce three variants of our method, namely TNPAR\_NT, TNPAR\_NC, and TNPAR\_MG. Specifically, TNPAR\_NT does not incorporate topological information and treats each event sequence independently. TNPAR\_MG also does not introduce the topological network and takes the event sequences in all nodes as one single sequence. TNPAR\_NC is a variant that excludes both acyclicity and sparsity regularization from the loss function. Detailed information about the baselines and variants can be found in Appendix C.

\subsection{Results on Simulated Data}
The F1 scores of different methods on the simulated data are depicted in Fig. \ref{f2}. Due to space limitations, results pertaining to Precision, Recall, SHD, and SID on the simulated data can be found in Appendix D. 

As shown in Fig. \ref{f2}, when compared with the results of PCMCI, ACD, ADM4, and MLE\_SGL (which do not consider the topological network behind the event sequences), both TNPAR and THP achieve superior performance across all cases. These results suggest that incorporating the topological network assists in learning Granger causality. Moreover, TNPAR and its variants significantly outstrip other baselines under most cases, demonstrating that the neural point process approach is powerful for modeling the generation process and the amortized inference technique is effective for Granger causal discovery. 

Traversing the value of parameters $\alpha$, $\mu$, and $\Delta$, the F1 scores of the TNPAR are generally better than those of other methods. These results signify that our method is relatively insensitive to the parameters of the generation process. Under settings with varying sample sizes, or changing numbers of event types or nodes, TNPAR achieves the highest F1 scores, illustrating the robustness of TNPAR. Even in cases with small sample sizes or high dimensional event types, TNPAR still outperforms other methods, corroborating the effectiveness of TNPAR.

Among the baseline methods, THP delivers the best results as it considers the topological network. Both ADM4 and MLE\_SGL outperform PCMCI and ACD, which demonstrates the efficacy of the point process methods in causal discovery for event sequences.

\subsection{Results on Metropolitan Cellular Network Alarm Data}
The experimental results of all algorithms on real-world data are presented in Table \ref{t: real world data}. From these results, it can be seen that the models incorporated the topological network (including TNPAR, TNPAR\_NC, THP) outperform that those rely on the i.i.d. assumption. Specifically, the F1 score of TNPAR is higher than that of THP, demonstrating that TNPAR can achieve the best trade-off between Precision and Recall without needing a prior distribution assumption. Although the Precision of TNPAR is marginally lower than that of THP, PCMCI, and ADM4, TNPAR exhibits the highest Recall. These results indicate that TNPAR is relatively insensitive to the weak Granger causal strength between event types, which accounts for the slightly higher SHD of TNPAR compared to other baselines. Moreover, TNPAR attains the best SID relative to other methods except ACD. These results underscore TNPAR's superior capability for causal inference \cite{peters2015structural}.

\begin{table}[]
\centering
\begin{tabular}{c|ccccc} 
\hline
Algorithm     & Prec.        & Rec.           & F1                  & SID  & SHD  \\ 
\hline
ACD           & 0.25          & 0.32          & 0.27           & \textbf{230.0}  & 127.0    \\
MLE\_SGL      & 0.39          & 0.12          & 0.18           & 263.0  & 100.0  \\
ADM4          & 0.49          & 0.34          & 0.40           & 292.0  & 95.0   \\
PCMCI         & 0.47          & 0.31          & 0.37           & 306.0  & 102.0  \\
THP           & \textbf{0.58}          & 0.35          & 0.44   & 266.0 & \textbf{85.0}       \\
TNPAR\_NT      & 0.31          & 0.48          & 0.38           &  257.0 &  150.4  \\
TNPAR\_NC      & 0.41          & 0.60          & 0.49           &  271.3 &  123.5  \\
TNPAR\_MG      & 0.35          & 0.31          & 0.33           &  272.5 &  115.2  \\
\textbf{TNPAR} & 0.46 & \textbf{0.65} & \textbf{0.54}  & 241.6   & 105.5  \\
\hline
\end{tabular}
\caption{Result on real-world data }
\label{t: real world data}
\end{table}

\subsection{Ablation Study}
Ablation experiments were conducted on both simulated and real-world data to assess the effectiveness of incorporating topological information and regularization in TNPAR. The experimental results given in Fig. \ref{f2} reveal a significant improvement in the performance of TNPAR compared to TNPAR\_NT, TNPAT\_MG, and TNPAR\_NC. These results underscore the importance of incorporating topological information, as well as acyclicity and sparsity regularization. Additionally, in most instances, TNPAR\_NT, TNPAR\_MG, and TNPAR\_NC perform comparably to THP, despite the fact that the THP fully leverages topological information and sparsity constraint. These results demonstrate the robustness of TNPAR.

\subsection{Case Study}
In the real-world data experiment, our method successfully infers the Granger causality across the nodes of the topological network, aligning with experts' knowledge. From the results, we observe that the count of causal edges diminishes as the geodesic distance, denoted as $k$, increases. This finding aligns with our intuition that some events are influenced solely by their node. For instance, the causal relationship $BD\_STATUS \rightarrow TU\_AIS$ is only valid for $k=0$. Here, $BD\_STATUS$ represents a device disconnecting from a node, and $TU\_AIS$ signifies that some tasks running on this node have been disrupted. Another notable discovery is that some edges exclusively transit across nodes, suggesting that these event types are triggered only by their neighbors. For example, consider $ETH\_LINK\_DOWN \rightarrow MW\_LOF$, where $ETH\_LINK\_DOWN$ denotes an error on the Ethernet interface in a downstream node, and $MW\_LOF$ indicates disconnection of an upstream node.

\section{Conclusion}
In this paper, we studied the Granger causal discovery problem on the topological event sequences. By leveraging the prior topological network and regarding the causal structure as a latent variable, we successfully addressed the challenge of unified modeling the topological network and latent causal structure, which enable us to propose the amortized causal discovery method. The experimental results on both simulation and real-world data demonstrate the effectiveness of our proposed method. TNPAR offers a viable solution for causal discovery in real-world settings with non-i.i.d. data. To further enhance its effectiveness, our future work will focus on improving sample efficiency, integrating prior knowledge, and so on.


\section{Acknowledgments}
This research was supported in part by  National Key R\&D Program of China (2021ZD0111501), National Science Fund for Excellent Young Scholars (62122022), Natural Science Foundation of China (61876043, 61976052, 62206061, 62206064), the major key project of PCL (PCL2021A12), Guangzhou Basic and Applied Basic Research Foundation (2023A04J1700). 

\bibliography{aaai24}

\clearpage
\noindent\textbf{Supplementary Material} \\

In this supplementary material, we summarize the methods related to our work in Appendix A. Then we provide the proof for Proposition 1 in Appendix B. After that, Appendix C describes the details of the baseline methods in experiments. At last, more results in simulated data are given in Appendix D.

\section*{A. Related Work Summarization}

\renewcommand\thefigure{A.\arabic{figure}}  
\renewcommand\thetable{A.\arabic{table}}  

In this section, we explicitly compare our proposed method with the related methods from four perspectives in Table \ref{t: compare with baseline method}. From this table, we can find that our method is the only method for non-i.i.d. topological event sequences.

\begin{table*}[h]
\centering
\caption{Related methods for causal discovery.}
\begin{tabular}{l|llll}
\hline
Methods  & Categories                 & Density Estimator        & Probabilistic Relation & Non-i.i.d. \\
\hline
MLE\_SGL & Hawkes process-based       & Gaussian-based kernel    & No                     & No         \\
ADM4     & Hawkes process-based       & exponential-based kernel & No                     & No         \\
THP      & Hawkes process-based       & exponential-based kernel & No                     & Yes        \\
PCMCI    & constraint-based           & -                        & No                     & No         \\
ACD      & auto-regressive-based      & -                        & Yes                    & Yes        \\
TNPAR    & neural point process-based & neural network           & Yes                    & Yes       \\
\hline
\end{tabular}

\label{t: compare with baseline method}
\end{table*}

\section*{B. Proof for Proposition 1}

\renewcommand\thefigure{B.\arabic{figure}}  
\renewcommand\thetable{B.\arabic{table}}  

\begin{proposition}
Given the Granger causal structure $\mathcal{G}_V(\mathbf{V}, \mathbf{E}_V)$, the topological network $\mathcal{G}_N(\mathbf{N}, \mathbf{E}_N)$ and the max geodesic distance $K$, along with the assumption that the data generation process adheres to Eq. (4), we can deduce that $v_i \rightarrow v_j \notin \mathbf{E}_V$, if and only if $A_{k}^{v_i, v_j} = 0$ for every  $k \in \{0,...,K\}$.
\label{proposition: TNPAR}
\end{proposition}

\begin{proof}
Let {$O_{t}^{v_j, n_k}$} and {$\hat{O}_{t}^{v_j, n_k}$} represent the outputs for $\mathrm{g}(\mathbf{O}^{v_1, n_1}_{t-1 : t-\Omega}, ..., \mathbf{O}^{v_i, n_1}_{t-1 : t-\Omega} ,..., \mathbf{O}^{v_i, n_{|N|}}_{t-1 : t-\Omega} ,..., \mathbf{O}^{v_{|V|}, n_{|N|}}_{t-1 : t-\Omega})$ and $\mathrm{g}(\mathbf{O}^{v_1, n_1}_{t-1 : t-\Omega}, ..., \mathbf{\hat{O}}^{v_i, n_1}_{t-1 : t-\Omega} ,..., \mathbf{\hat{O}}^{v_i, n_{|N|}}_{t-1 : t-\Omega} ,..., \mathbf{O}^{v_{|V|}, n_{|N|}}_{t-1 : t-\Omega})$ in Definition 3, respectively. 

\noindent$\Longrightarrow$:
If $v_i\rightarrow v_j \notin \mathbf{E}_V$, then according to Definition 3 and Eq. (4), it necessitates that $O_t^{v_j, n_k} = \hat{O}_t^{v_j, n_k}$. Conversely, according to Eq. (4), if $A_{k}^{v_i, v_j} \neq 0$ for any  $k \in \{0,...,K\}$, the inputs for $\mathrm{g}(\cdot)$ will differ, potentially altering its output, (i.e., {\small$O_{t}^{v_j, n_k} \neq \hat{O}_{t}^{v_j, n_k}$}), leading to a contradiction. Hence, if $v_i\rightarrow v_j \notin \mathbf{E}_V$, it logically follows that $A_{k}^{v_i, v_j} = 0$ for every  $k \in \{0,...,K\}$ must hold.

\noindent$\Longleftarrow $:
If $A_{k}^{v_i, v_j} = 0$ for every  $k \in \{0,...,K\}$, then {$O_{t}^{v_j, n_k}$} and {$\hat{O}_{t}^{v_j, n_k}$} will yield identical values, as their corresponding input values for Eq. (4) remain unchanged following the application of the filter function $\mathrm{f}^{\mathcal{G}_V}_{v_i}(\cdot)$. This implies that if $A_{k}^{v_i, v_j} = 0$ for every  $k \in \{0,...,K\}$, then {$O_{t}^{v_j, n_k}$}  remains unaffected by varying $\mathbf{O}^{v_i, n_1}_{t-1 : t-\Omega} ,..., \mathbf{O}^{v_i, n_{|N|}}_{t-1 : t-\Omega}$, allowing us to infer $v_i\rightarrow v_j \notin \mathbf{E}_V$ as per Definition 3.
\end{proof}

\section*{C. Details of the Methods in Experiments}

\renewcommand\thefigure{C.\arabic{figure}}  
\renewcommand\thetable{C.\arabic{table}}  

This section provides the details and parametric settings of methods chosen by cross-validation in the experiments. We first begin with the baseline methods as follows:
\begin{itemize}[leftmargin=*]
\item \textbf{PCMCI:} The algorithm uses the independence test in the first stage to remove the irrelevant edges between event types at different time lags from a complete undirected graph. Subsequently, the momentary conditional independence (MCI) test is introduced in the second stage to determine the arrows for the undirected edges. In these experiments, we set the time interval to 5 seconds, the max time lag to 2, and the significant level of independence test to 0.01. Because the input data is discrete, the conditional mutual information (CMI) test is used as the independence test method. 
\item \textbf{MLE\_SGL:} The algorithm combines a maximum likelihood estimator (MLE) with a sparse group lasso (SGL) regularizer for the Hawkes process model to learn the Granger causal structure between event types. In detail, MLE\_SGL uses a linear combination of Gaussian base functions to implement the intensity function in the Hawkes process. 
In these experiments, the number of the Gaussian base function is set to 10 and the coefficient of the group lasso regularization is set to 0.5.

\item \textbf{ACD:} The algorithm introduces an amortized model that infers causal relations across time series with different underlying causal graphs, which share relevant information. Specifically, ACD encompasses an amortized encoder that predicts the edges in the causal graph, and a decoder models the dynamics of the system under the predicted causal relations. In these experiments, we aggregate the event sequences into time series using a window size of 5 seconds. For the algorithm's execution, we employ the original MLP encoder and decoder, utilizing the default parameters given by the author.

\item \textbf{ADM4:} The algorithm learns the Granger causal structure between event types by combining techniques of alternating direction method of multipliers and majorization minimization for Hawkes process. Different from MLE\_SGL, it uses the exponential kernel to implement the intensity function with a decay parameter $\beta$. 
To impose priors on sparsity and low-rank causal structure, both the nuclear norm and $\ell_1$-norm regularization are adopted. In our experiment, the decay $\beta$ is set to 0.4 and the coefficients of $\ell_1$-norm and nuclear regularization are set to 0.05.
\item \textbf{THP:} The algorithm extends the time-domain convolution property of the Hawkes process to the time-graph domain for learning causal structure among event types generated in a topological network. 
In these experiments, we follow the original implementation of THP and use an exponential kernel to model its intensity function. The decay rate $\delta$ of the exponential kernel is set to 0.1.
\end{itemize}

To evaluate the modules of the proposed method, we devise three variants of our method: TNPAR\_NT, TNPAR\_MG, and TNPAR\_NC. The details of our method and its variants in the experiments are as follows.
\begin{itemize}[leftmargin=*]
\item \textbf{TNPAR:} In this work, we propose TNPAR for Granger causal discovery on topological event sequences. By incorporating the topological network $\mathbf{B}_{0:K}$ and the casual structure $\mathbf{A}_{0:K}$, TNPAR solves the problem of causal discovery from non-i.i.d. data. In addition, TNPAR is applied to optimize the parameters with a sparsity regularization $\mathcal{L}_{s}$ and an acyclicity regularization $\mathcal{L}_{c}$ on $\mathbf{A}_{0:K}$. In these experiments, to obtain a sparse and acyclic causal structure, $\mathcal{L}_s$ is set to 1e-4 and $\mathcal{L}_c$ is 1e-10.
\item \textbf{TNPAR\_NT:} To evaluate the effectiveness of introducing topological neighbors' information into TNPAR, TNPAR\_NT only uses the historical information on its own node, without considering the topological information.
In our experiment, $\lambda_{s}$ is set to 1e-4 and $\lambda_{c}$ is set to 1e-10.

\item \textbf{TNPAR\_MG:} To evaluate the effectiveness of introducing the topological network into TNPAR, we merge all event sequences produced by the topological nodes, which is used as the training data for TNPAR\_MG. In our experiment, $\lambda_{s}$ is set to 1e-9 and $\lambda_{c}$ is set to 1e-7.

\item \textbf{TNPAR\_NC:} To evaluate the effectiveness of the acyclic and sparse regularization in TNPAR, TNPAR\_NC removes the acyclicity and sparsity regularization from the loss function. That is, $\lambda_{s}$ and $\lambda_{c}$ are both set to 0.

\end{itemize}

\section*{D. Details of the Results in Simulation Dataset}

\renewcommand\thefigure{D.\arabic{figure}}  
\renewcommand\thetable{D.\arabic{table}}  

In this section, we further provide F1, Precision, Recall, SID, and SHD of baseline methods and our method in each experimental setting. The results are given in Fig. \ref{appendix fig: f1} - \ref{appendix fig: SHD}. All experiments are conducted on 60 cores AMD EPYC 7543 Processor and 4 * RTX 3090 with 320GB RAM. The operating system is Linux. In addition, the relevant software libraries and frameworks include the PyTorch 1.7.0 version, Python 3.8 version, and Cuda 11.0.

From the experimental results, we can find that TNPAR achieves the best performance compared to all baselines (MLE\_SGL\, ADM4, PCMCI, and THP), which demonstrates the superiority of TNPAR. In addition, with the varying parameters $\alpha$, $\mu$, and $\Delta$ in the data generation process, the performance of TNPAR is more robust than that of the baselines in all cases. The reason is that TNPAR is implemented by the neural point process, which holds a powerful capacity in modeling the generation process. Moreover, even in the cases of low sample size, high event number, or high node number, the experimental results of TNPAR remain well, which demonstrates its excellence.

Compared with the variants (TNPAR\_NT, TNPAR\_MG, and TNPAR\_NC), TNPAR achieves the best F1 score and SHD in all ablation studies, which proves the necessity of introducing the topological network, sparse constraint, and acyclic constraint. The Recall of TNPAR\_NT and TNPAR\_MG are lower than that of TNPAR because they do not learn the causal relationships across nodes while TNPAR does. In addition, the Precision of TNPAR\_NC is lower than that of TNPAR, which means there are more redundant edges in the result of TNPAR\_NC than TNPAR due to the lack of constraints.


\begin{figure*}[]
	\centering
	\subfigure[Sensitivity to Range of $\alpha (\times 1e-2)$]{
		\includegraphics[width=0.45\textwidth]{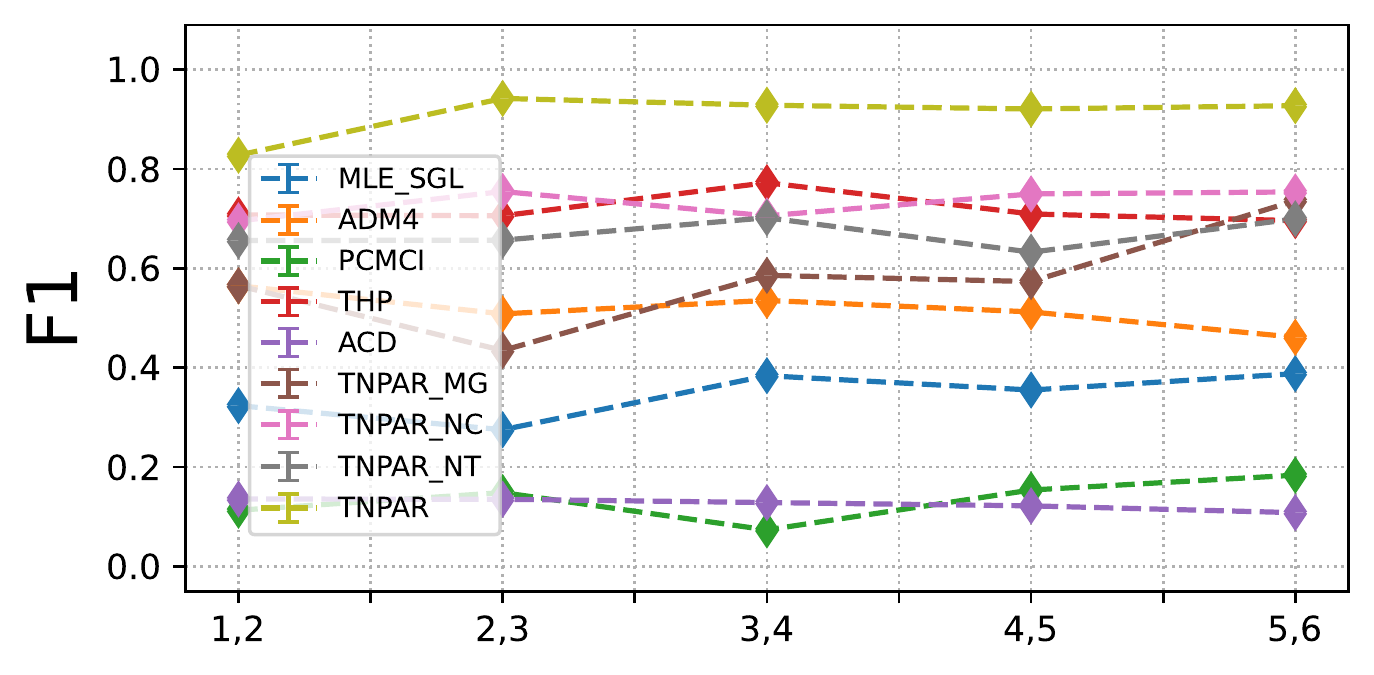}
		\label{appendix fig f1:a}
	}
	\subfigure[Sensitivity to Range of $\mu (\times 1e-5)$]{
	\includegraphics[width=0.45\textwidth]{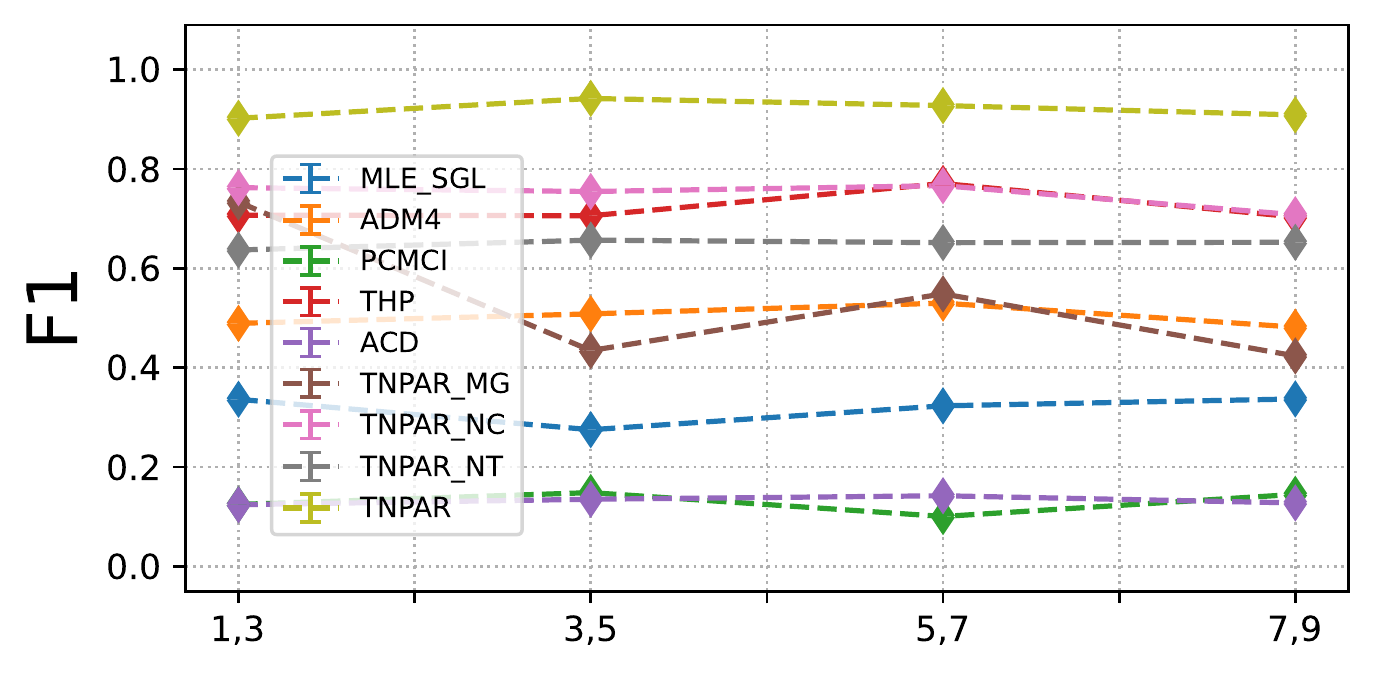}
	\label{appendix fig f1:b}
}
	\subfigure[Sensitivity to Sample Size]{
	\includegraphics[width=0.45\textwidth]{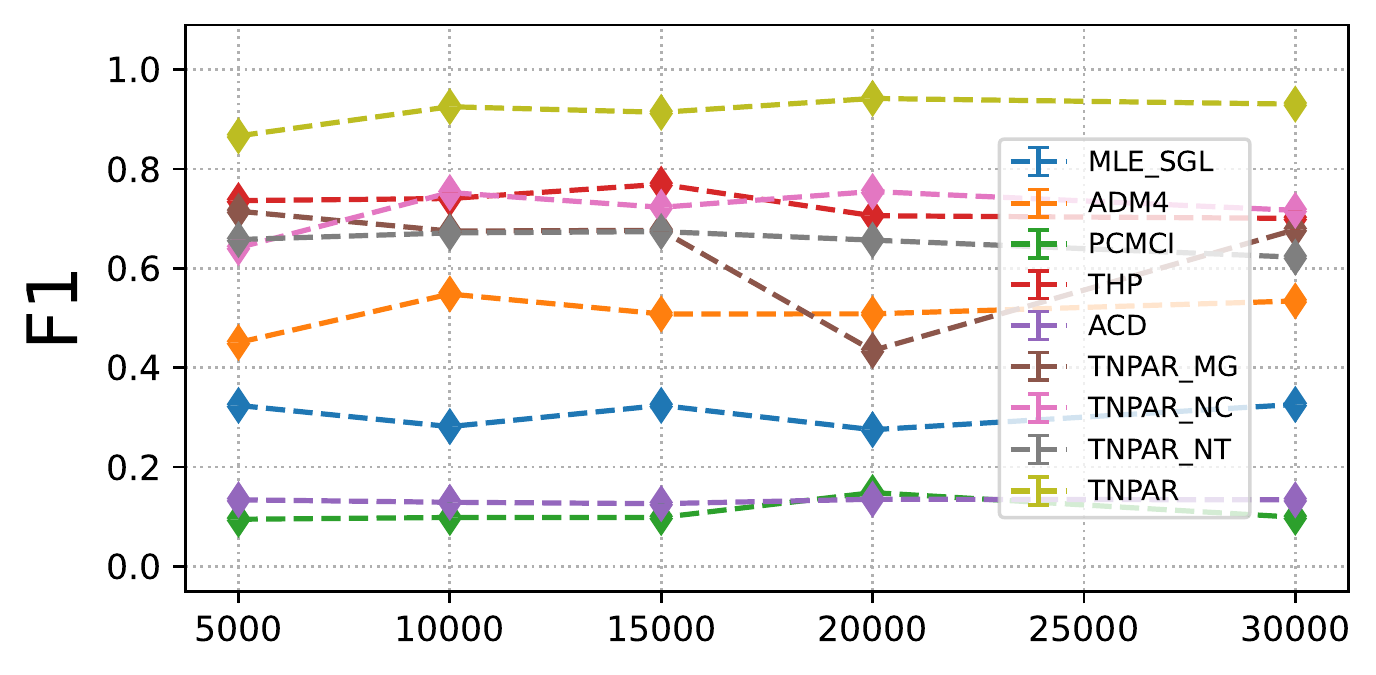}
	\label{appendix fig f1:c}
}
	\subfigure[Sensitivity to $\Delta$]{
	\includegraphics[width=0.45\textwidth]{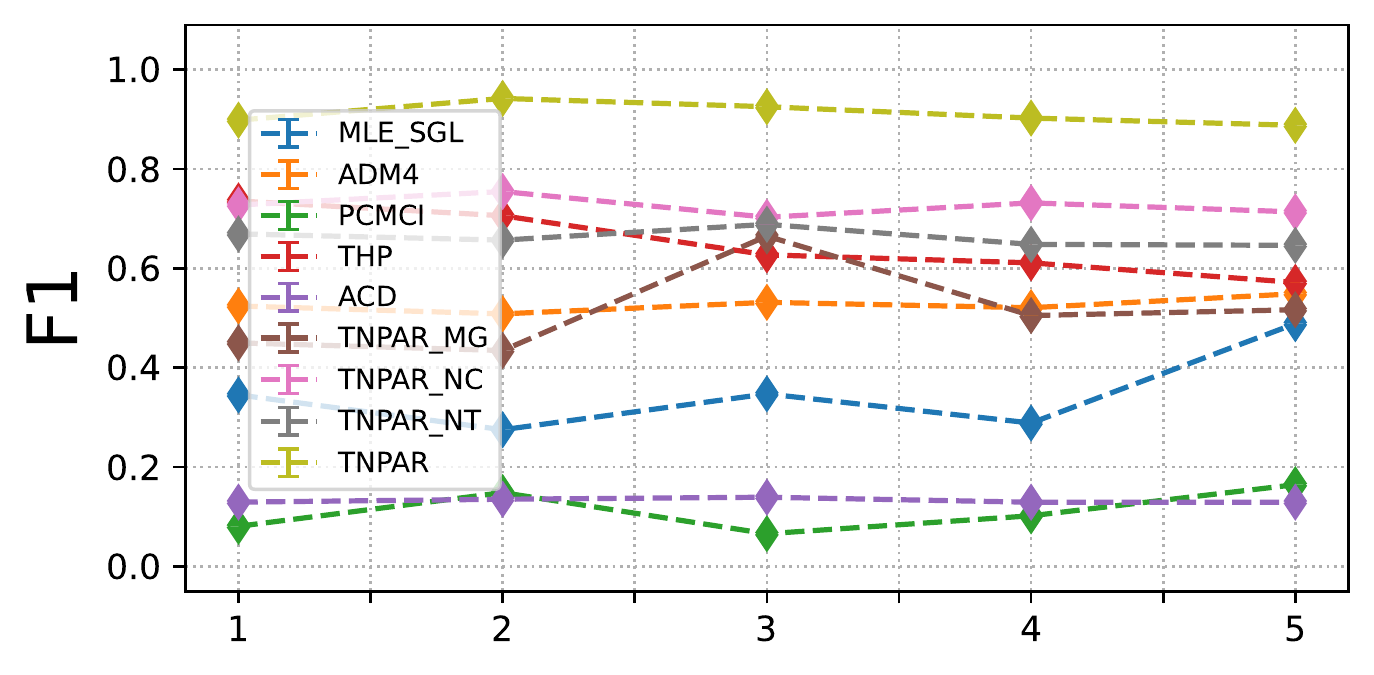}
	\label{appendix fig f1:d}
}
	\subfigure[Sensitivity to Num. of Event Types]{
	\includegraphics[width=0.45\textwidth]{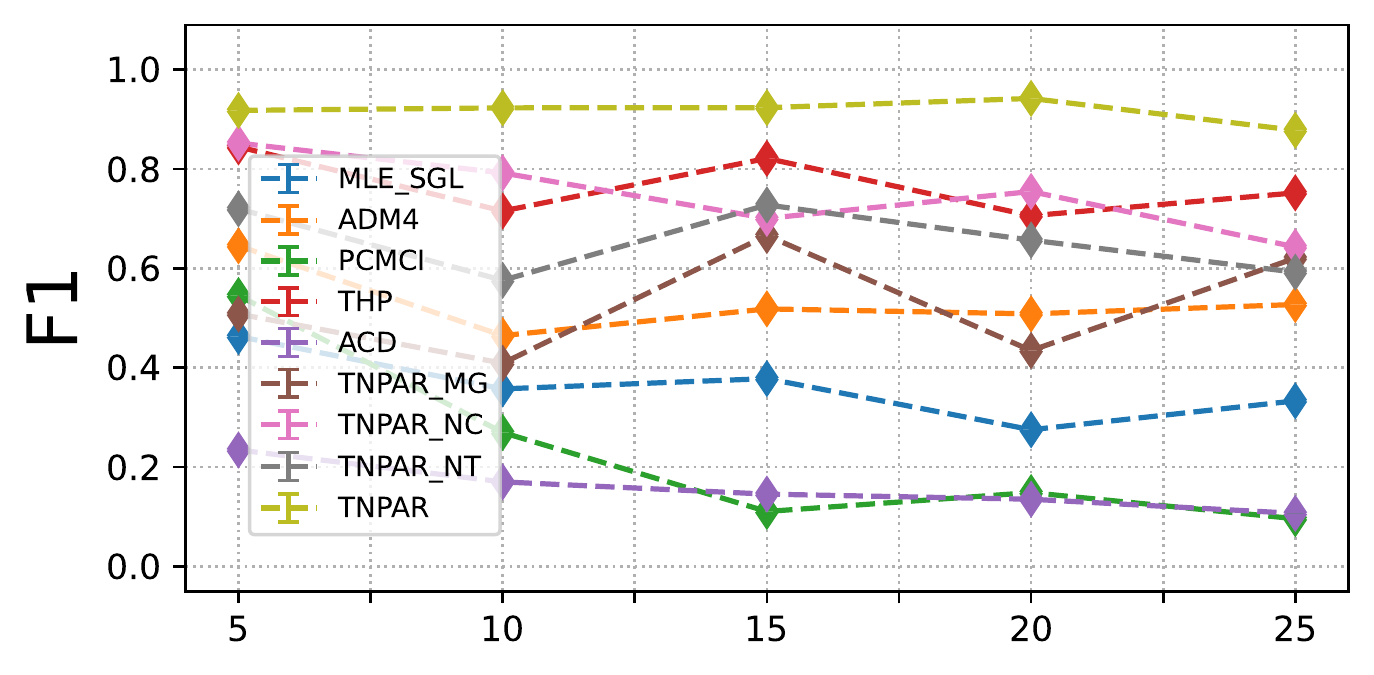}
	\label{appendix fig f1:e}
}
	\subfigure[Sensitivity to Num. of Nodes]{
	\includegraphics[width=0.45\textwidth]{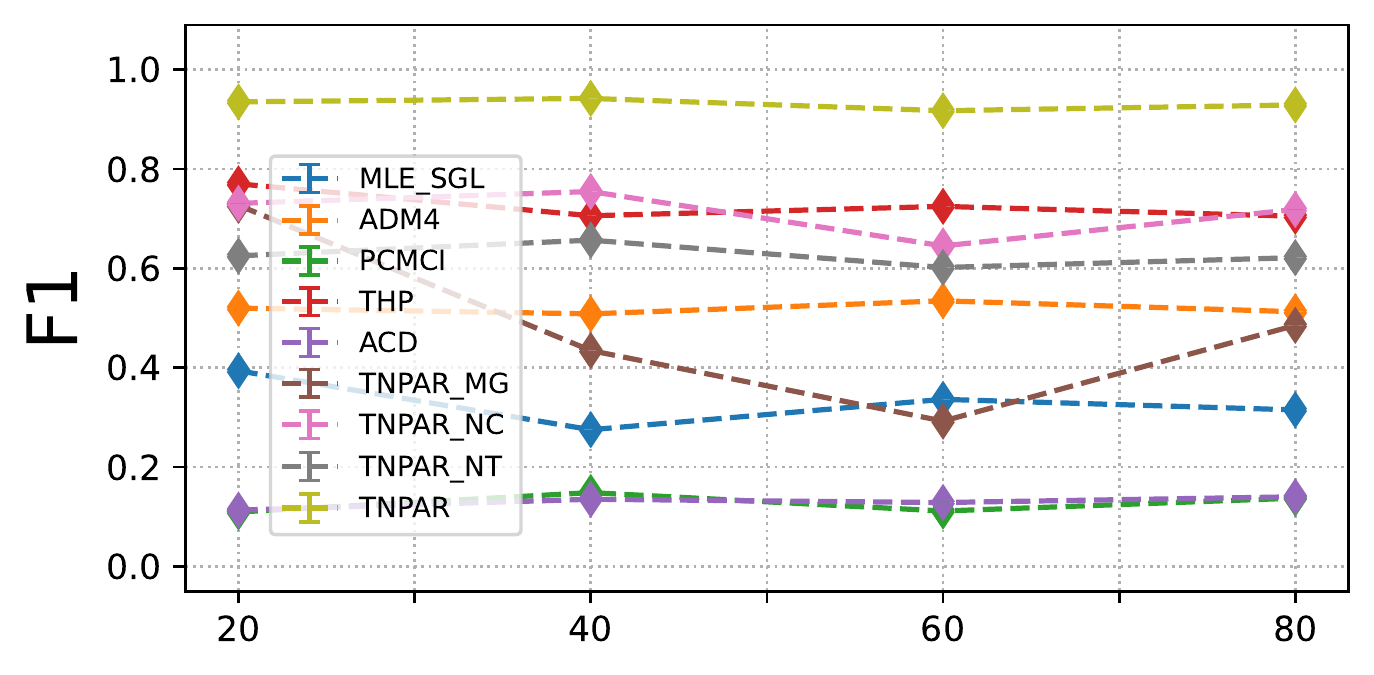}
	\label{appendix fig f1:f}
}
	\caption{F1 scores on the simulated data}	
	\label{appendix fig: f1}
\end{figure*}

\begin{figure*}[]
	\centering
	\subfigure[Sensitivity to Range of $\alpha (\times 1e-2)$]{
		\includegraphics[width=0.45\textwidth]{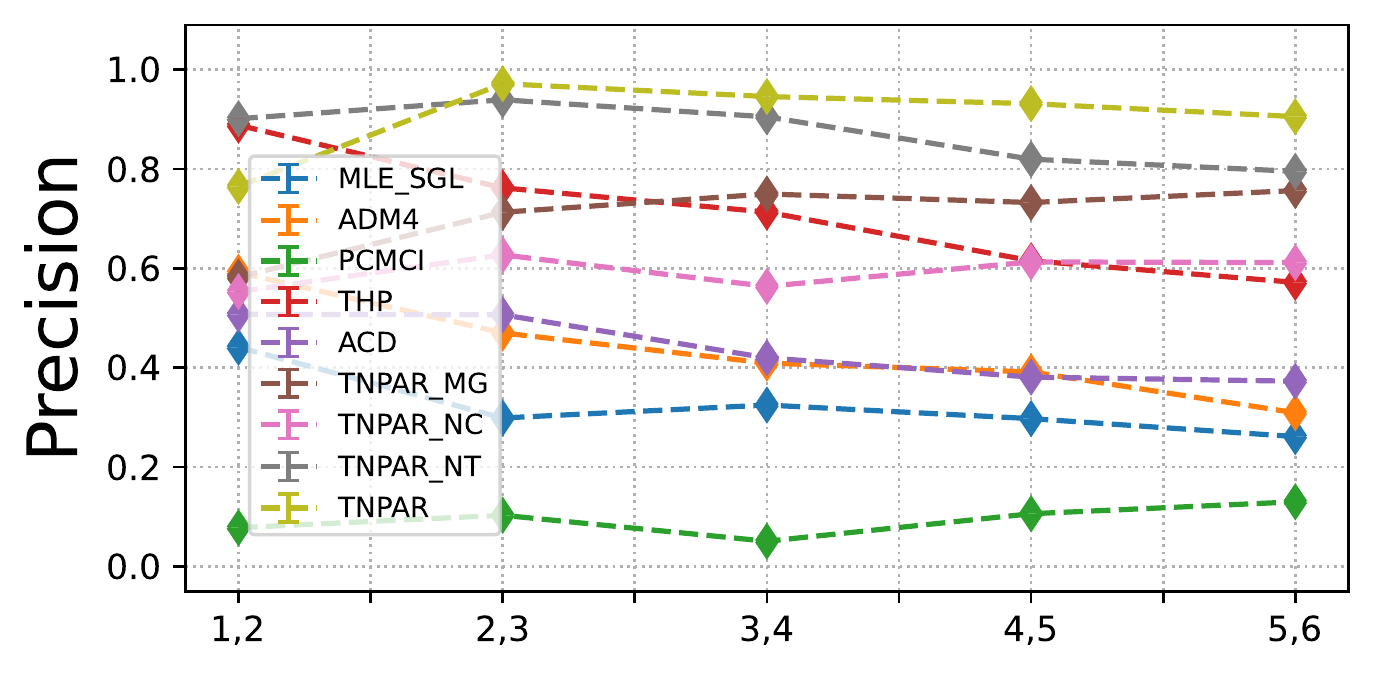}
		\label{appendix fig precision:a}
	}
	\subfigure[Sensitivity to Range of $\mu (\times 1e-5)$]{
	\includegraphics[width=0.45\textwidth]{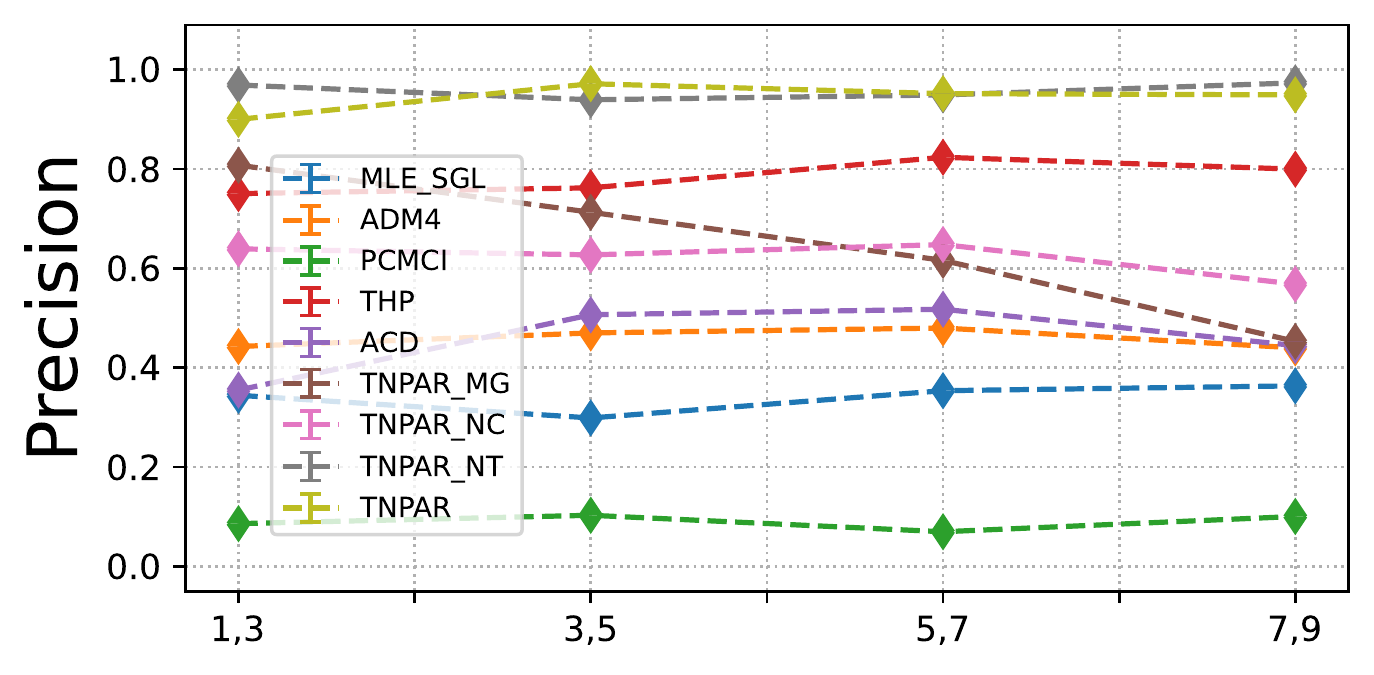}
	\label{appendix fig precision:b}
}
	\subfigure[Sensitivity to Sample Size]{
	\includegraphics[width=0.45\textwidth]{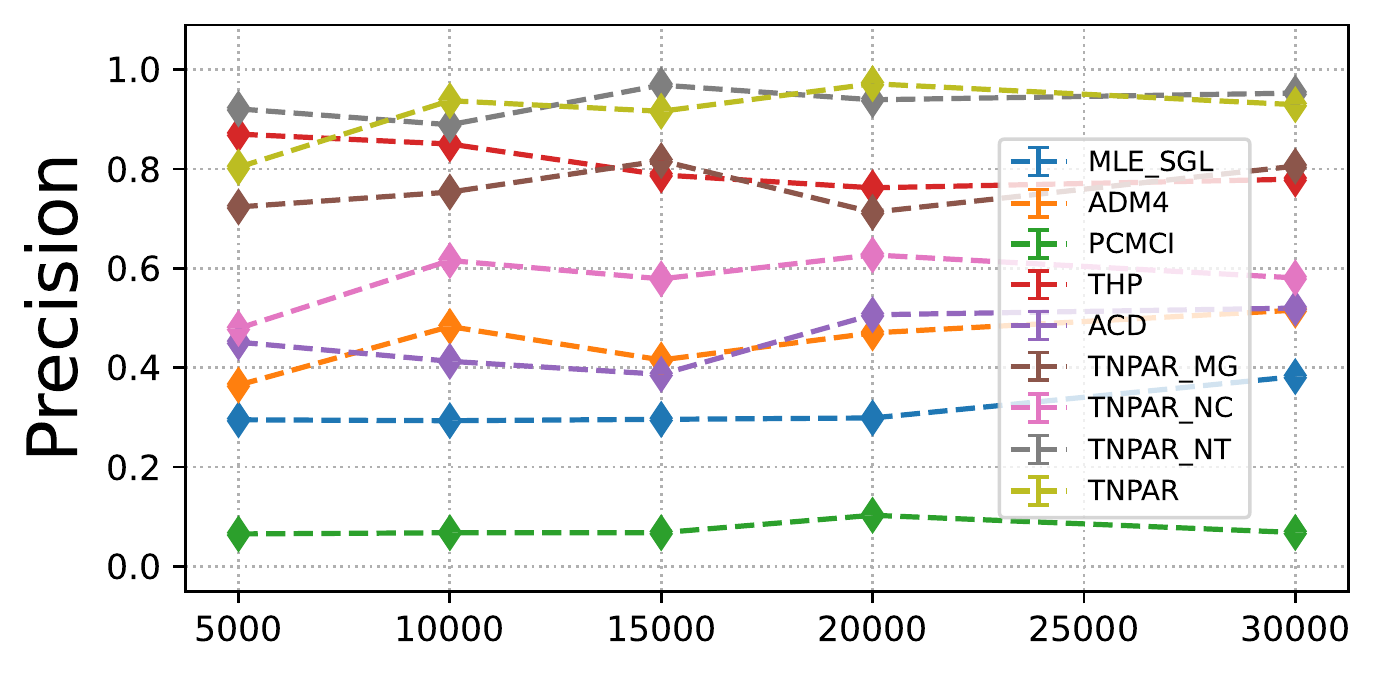}
	\label{appendix fig precision:c}
}
	\subfigure[Sensitivity to $\Delta$]{
	\includegraphics[width=0.45\textwidth]{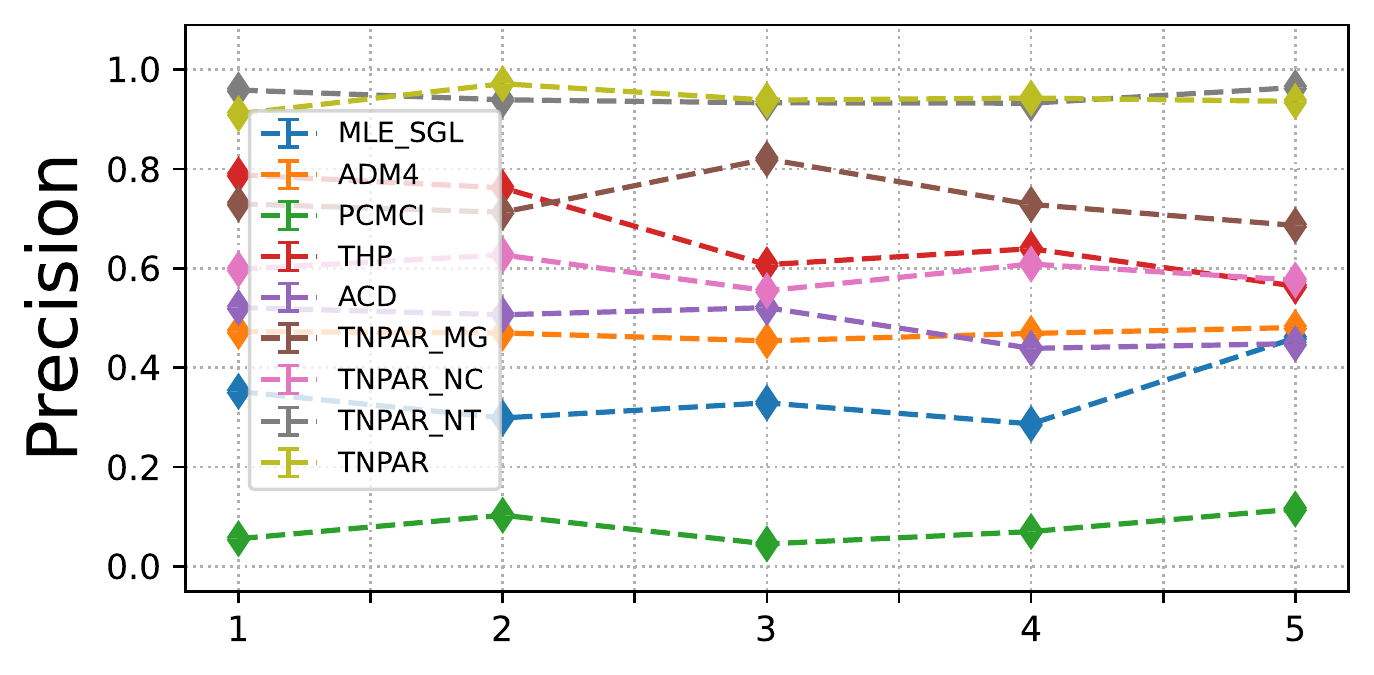}
	\label{appendix fig precision:d}
}
	\subfigure[Sensitivity to Num. of Event Types]{
	\includegraphics[width=0.45\textwidth]{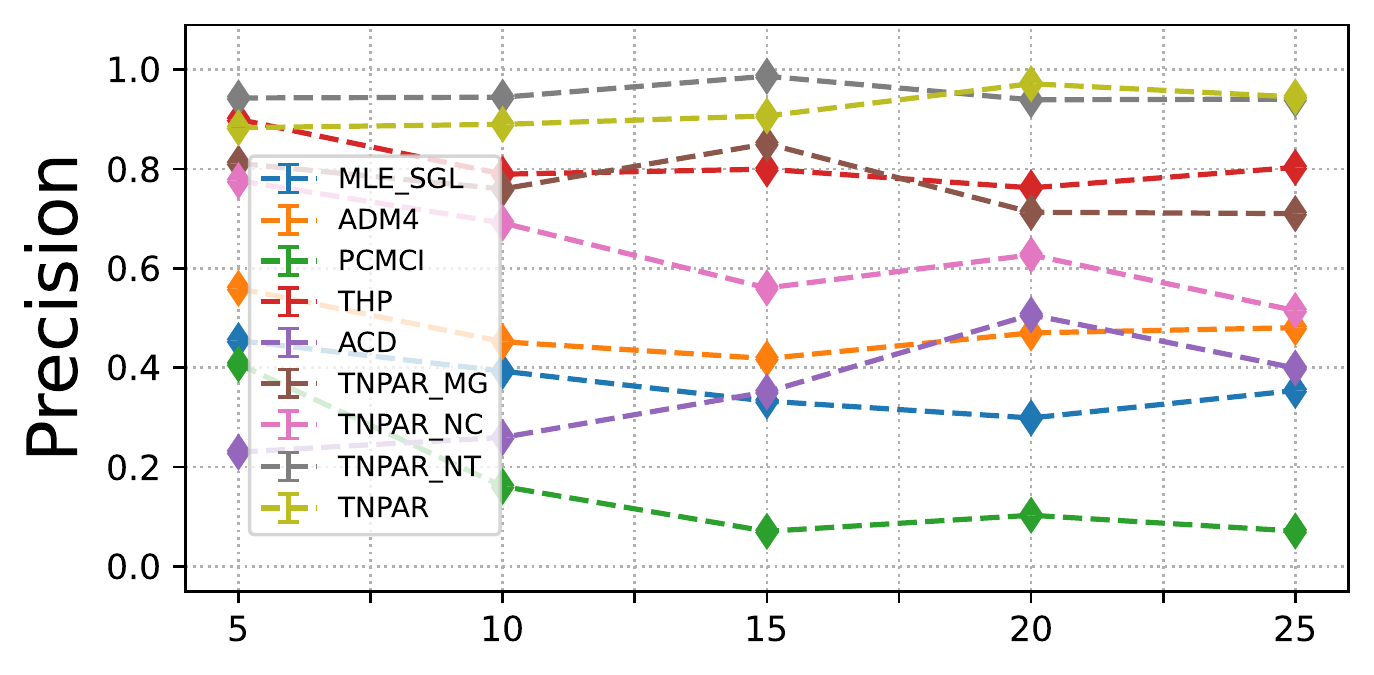}
	\label{appendix fig precision:e}
}
	\subfigure[Sensitivity to Num. of Nodes]{
	\includegraphics[width=0.45\textwidth]{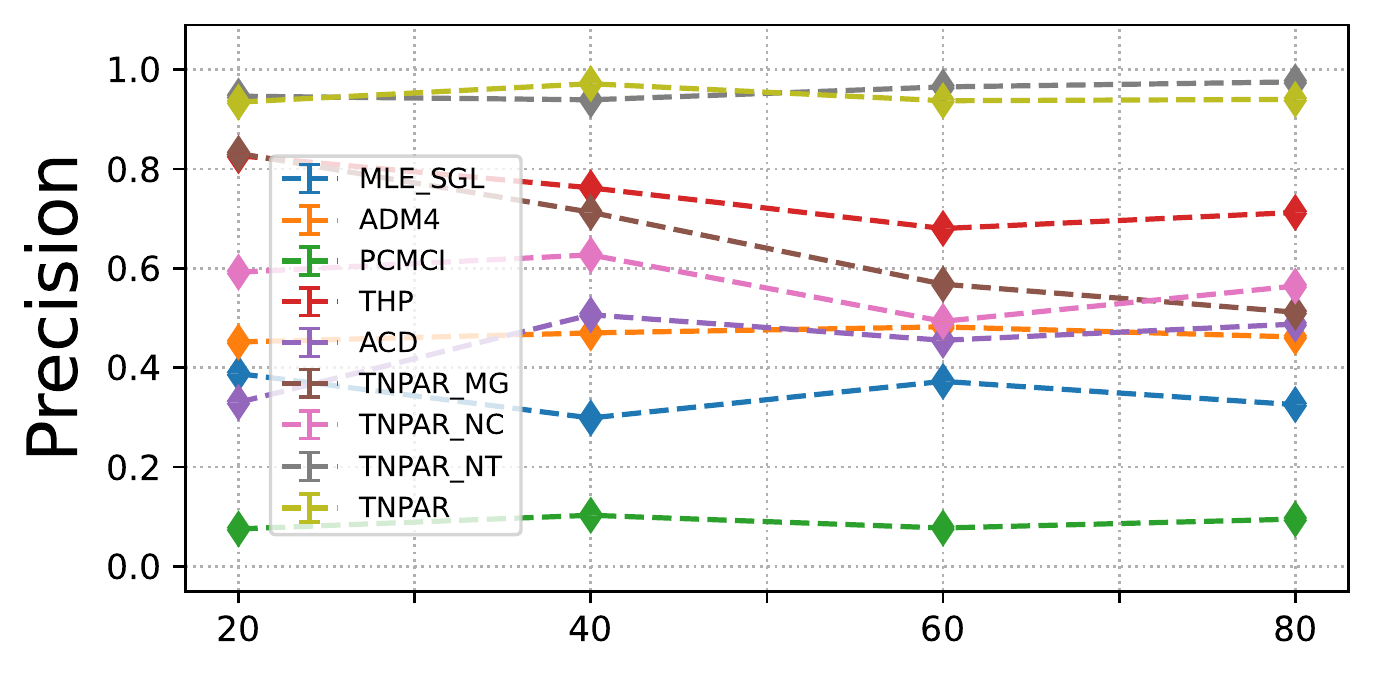}
	\label{appendix fig precision:f}
}
	\caption{Precision on the simulated data}	
	\label{appendix fig: precision}
\end{figure*}

\begin{figure*}[]
	\centering
	\subfigure[Sensitivity to Range of $\alpha (\times 1e-2)$]{
		\includegraphics[width=0.45\textwidth]{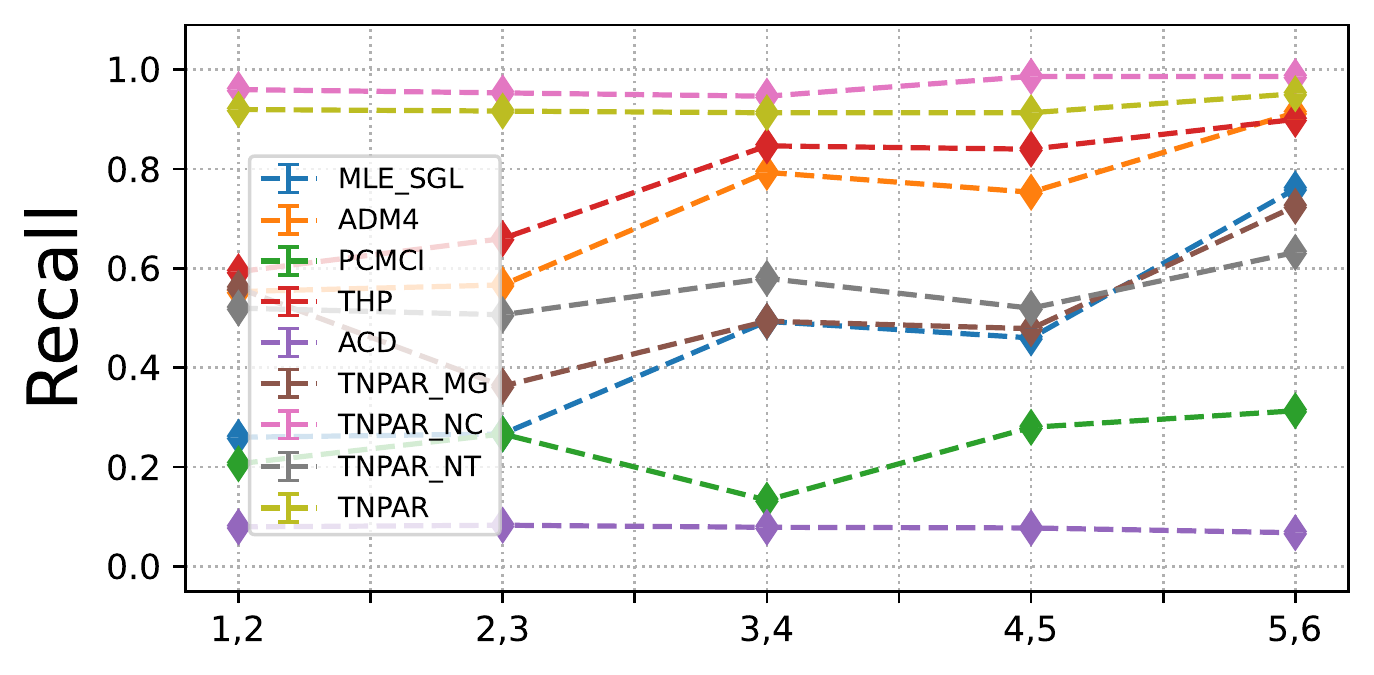}
		\label{appendix fig recall:a}
	}
	\subfigure[Sensitivity to Range of $\mu (\times 1e-5)$]{
	\includegraphics[width=0.45\textwidth]{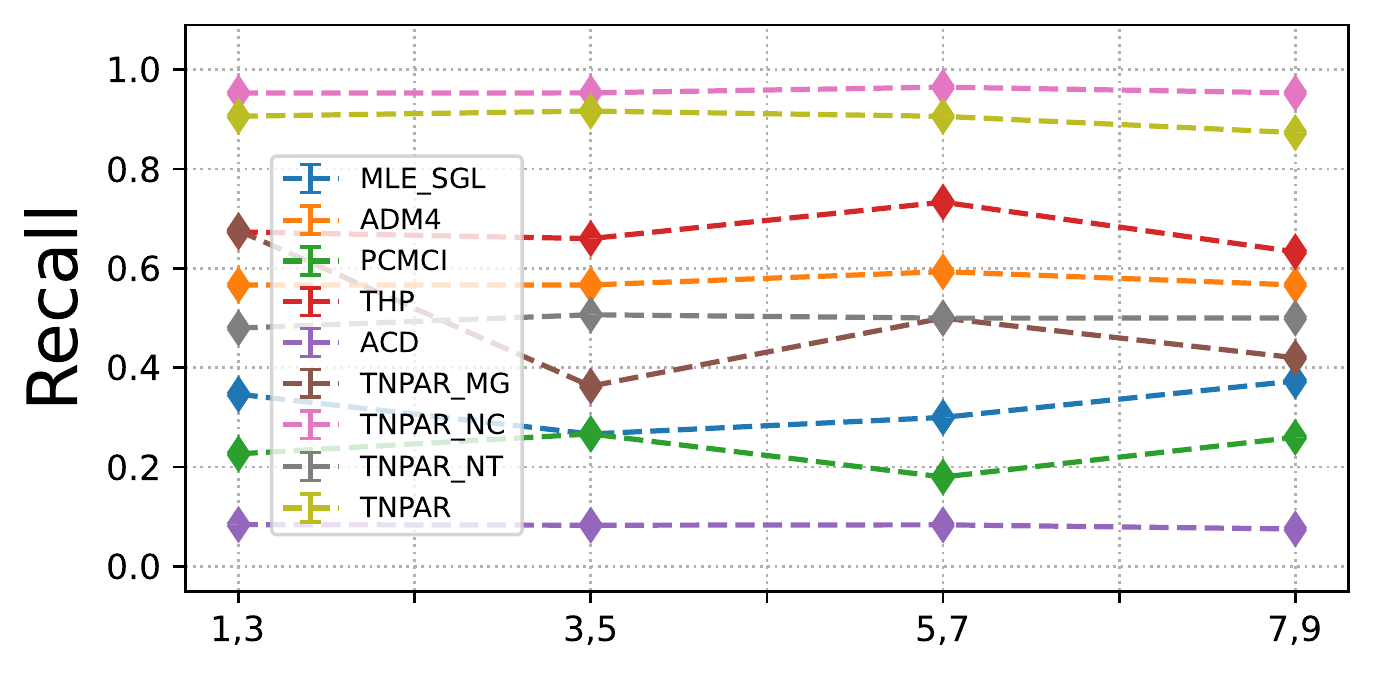}
	\label{appendix fig recall:b}
}
	\subfigure[Sensitivity to Sample Size]{
	\includegraphics[width=0.45\textwidth]{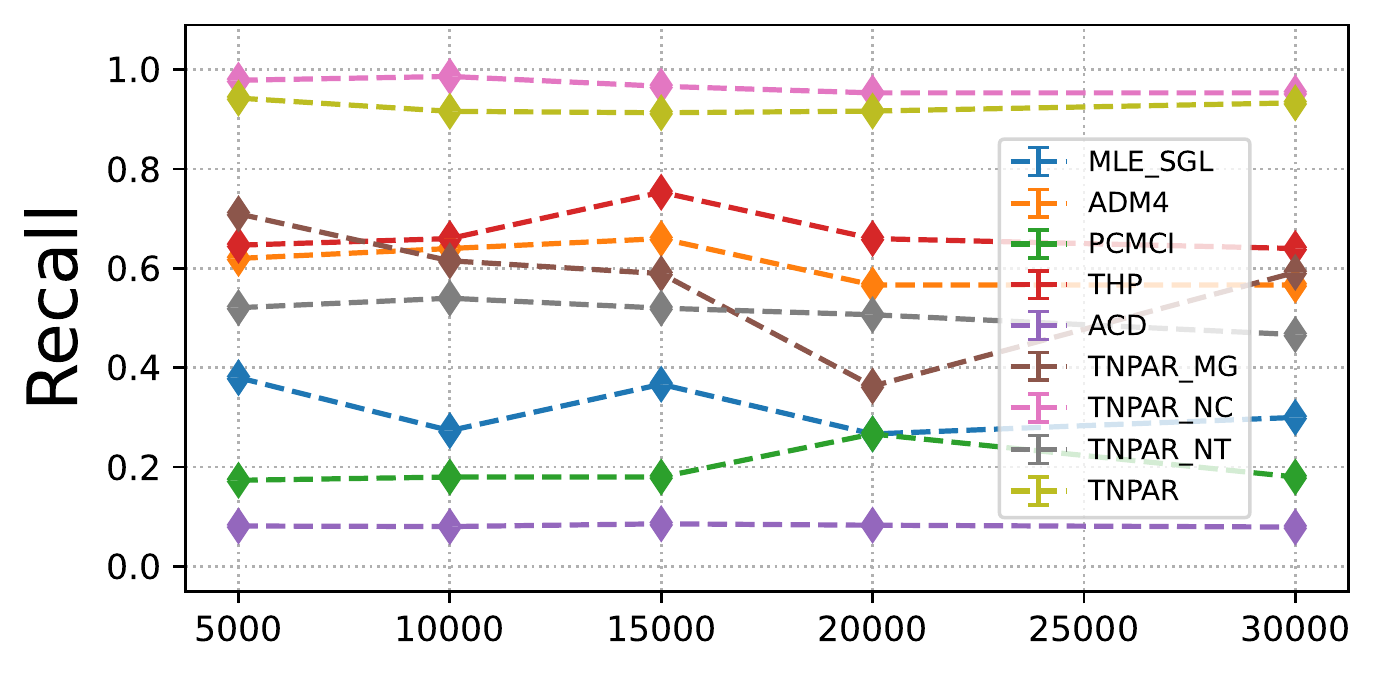}
	\label{appendix fig recall:c}
}
	\subfigure[Sensitivity to $\Delta$]{
	\includegraphics[width=0.45\textwidth]{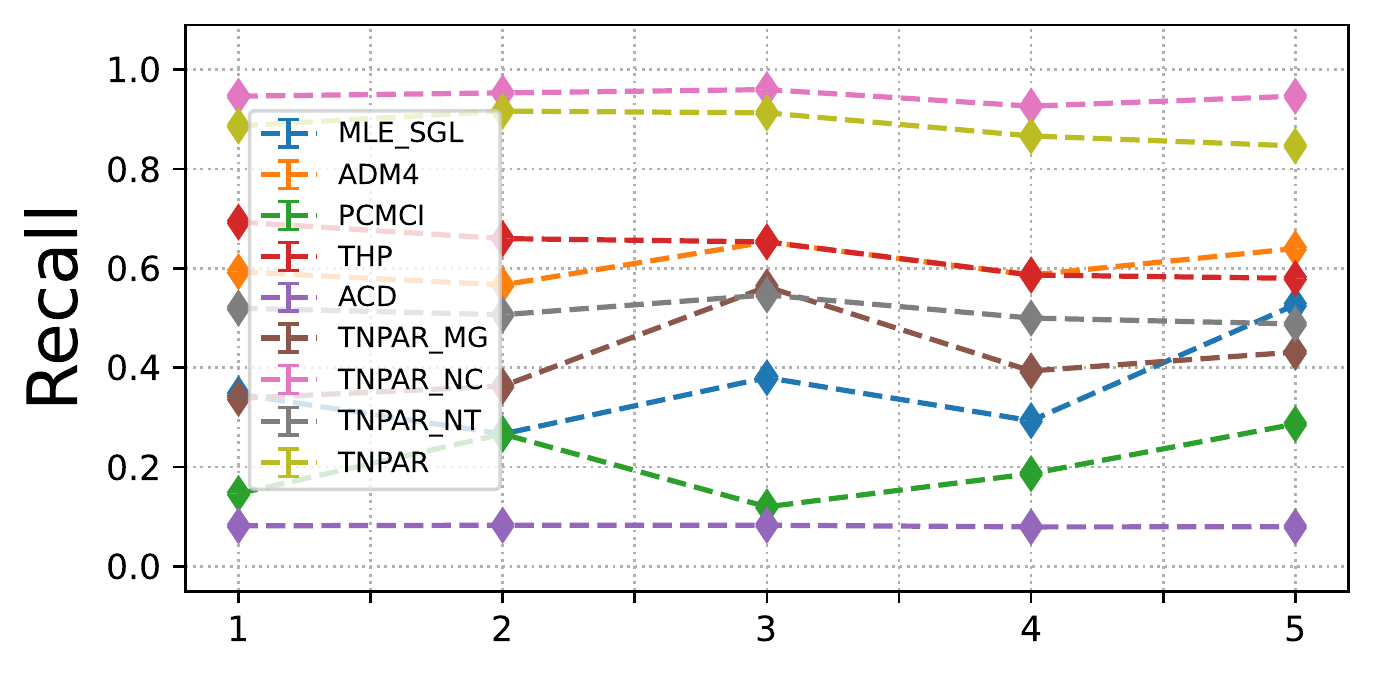}
	\label{appendix fig recall:d}
}
	\subfigure[Sensitivity to Num. of Event Types]{
	\includegraphics[width=0.45\textwidth]{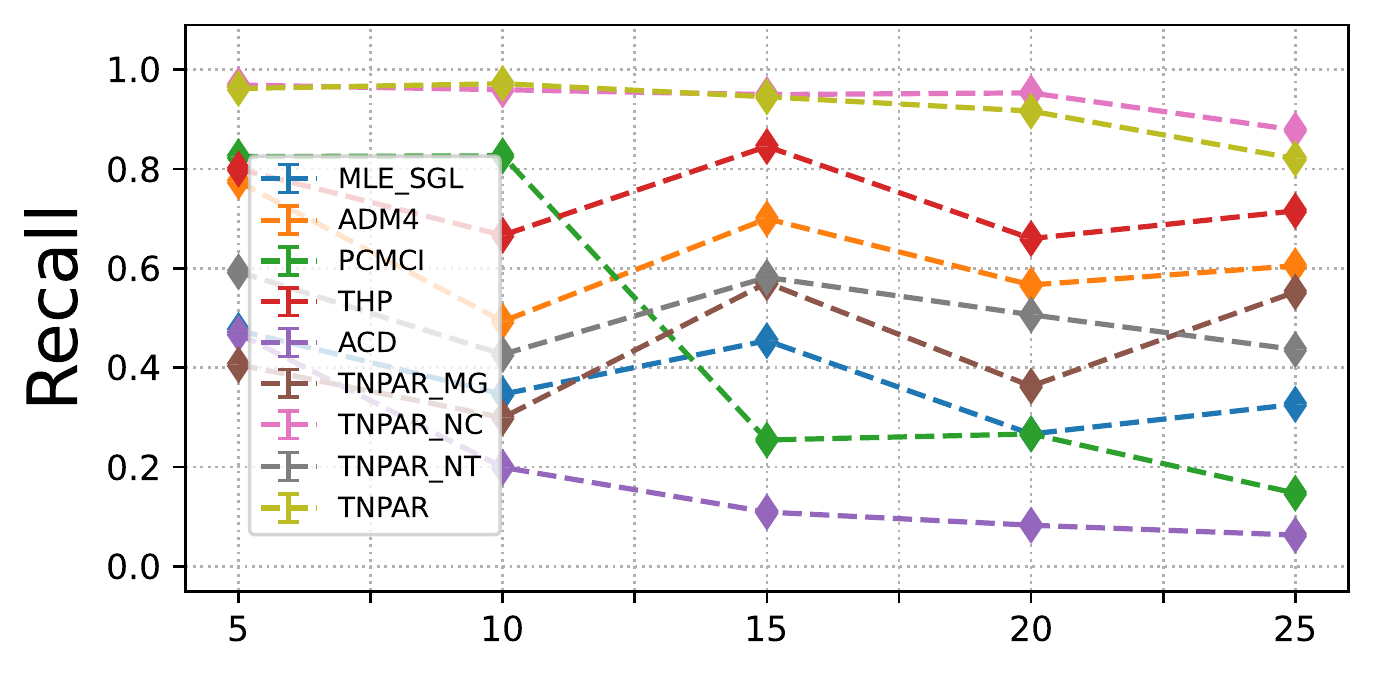}
	\label{appendix fig recall:e}
}
	\subfigure[Sensitivity to Num. of Nodes]{
	\includegraphics[width=0.45\textwidth]{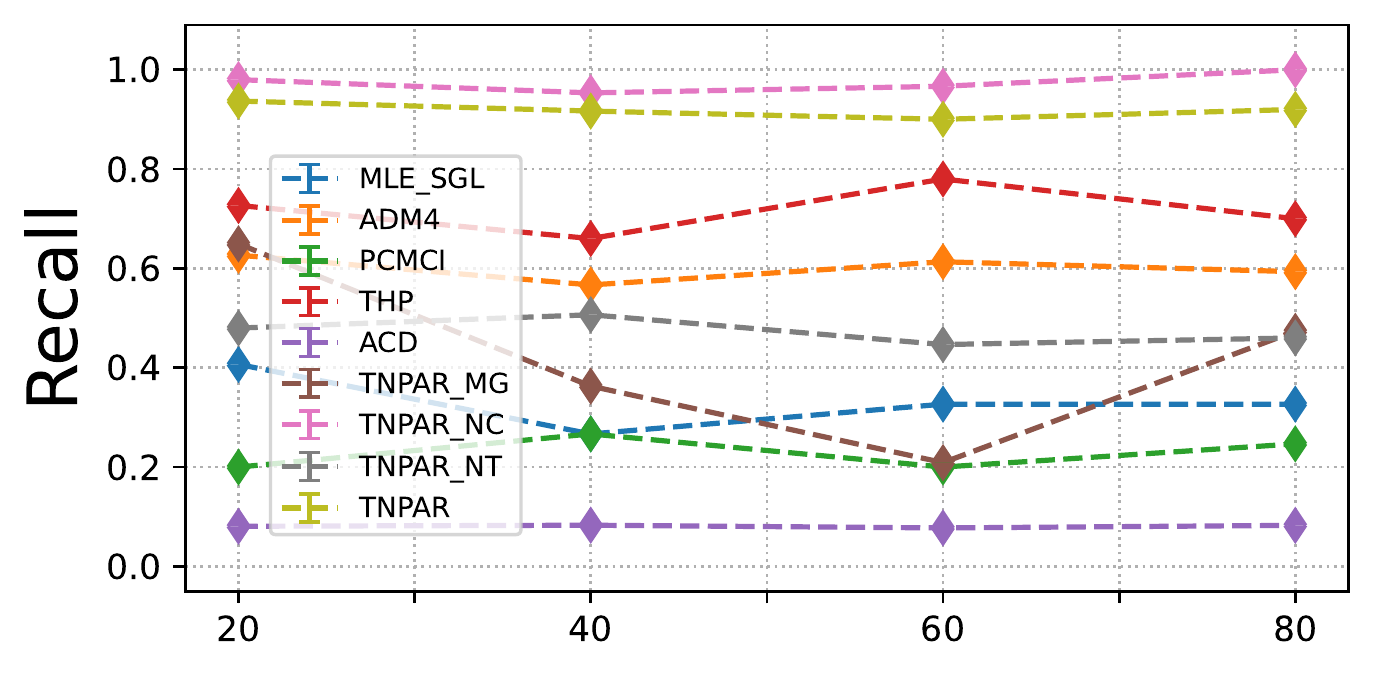}
	\label{appendix fig recall:f}
}
	\caption{Recall on the simulated data}	
	\label{appendix fig: recall}
\end{figure*}

\begin{figure*}[]
	\centering
	\subfigure[Sensitivity to Range of $\alpha (\times 1e-2)$]{
		\includegraphics[width=0.45\textwidth]{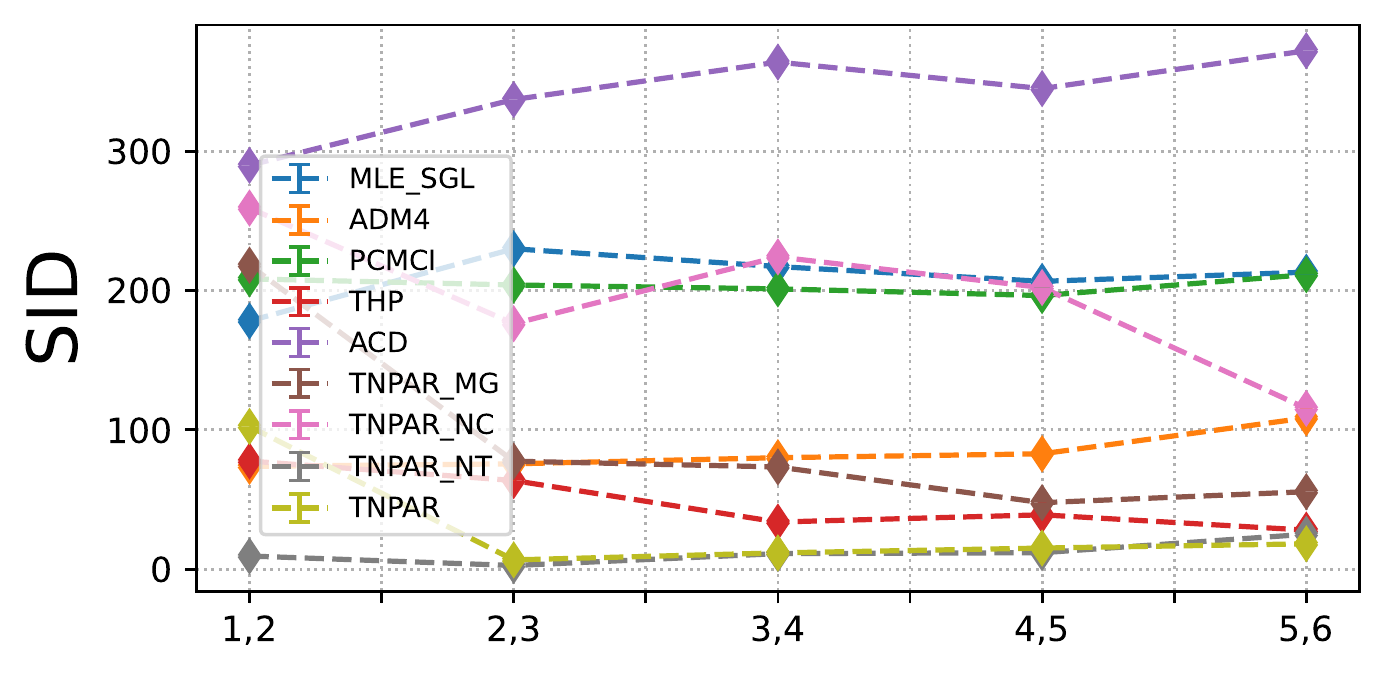}
		\label{appendix fig sid:a}
	}
	\subfigure[Sensitivity to Range of $\mu (\times 1e-5)$]{
	\includegraphics[width=0.45\textwidth]{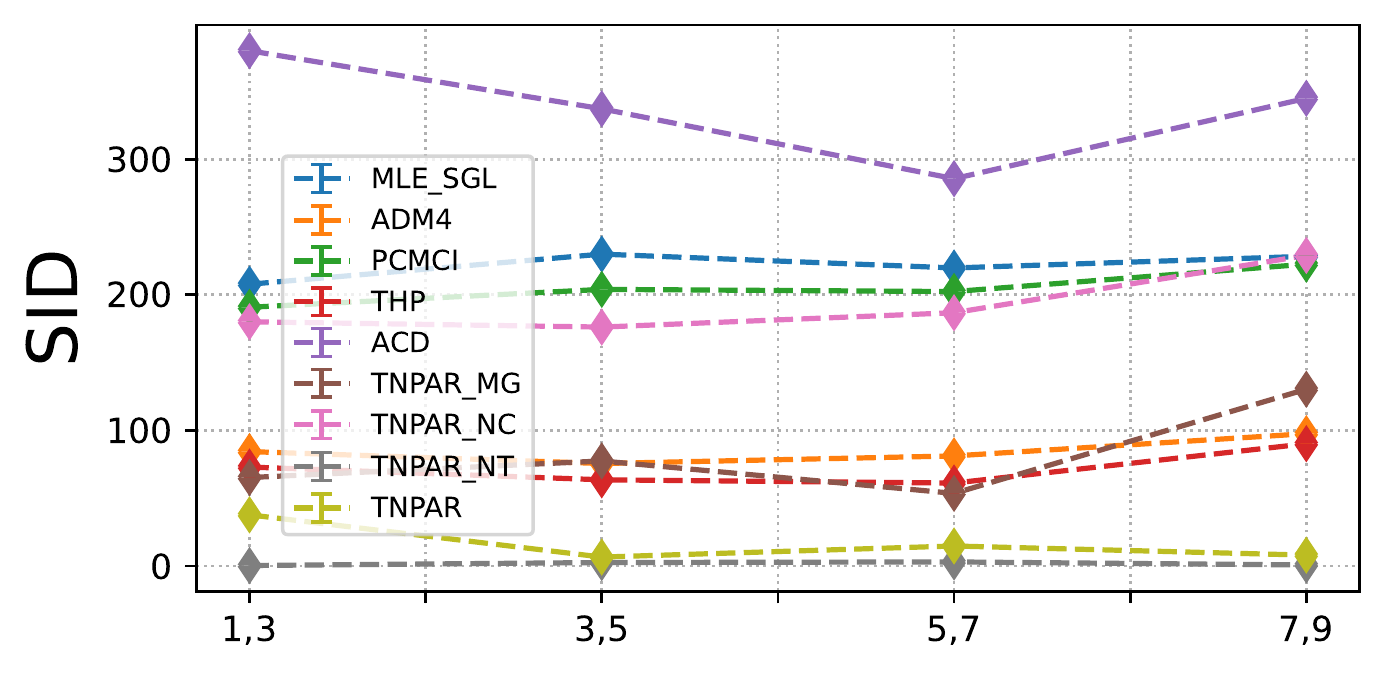}
	\label{appendix fig sid:b}
}
	\subfigure[Sensitivity to Sample Size]{
	\includegraphics[width=0.45\textwidth]{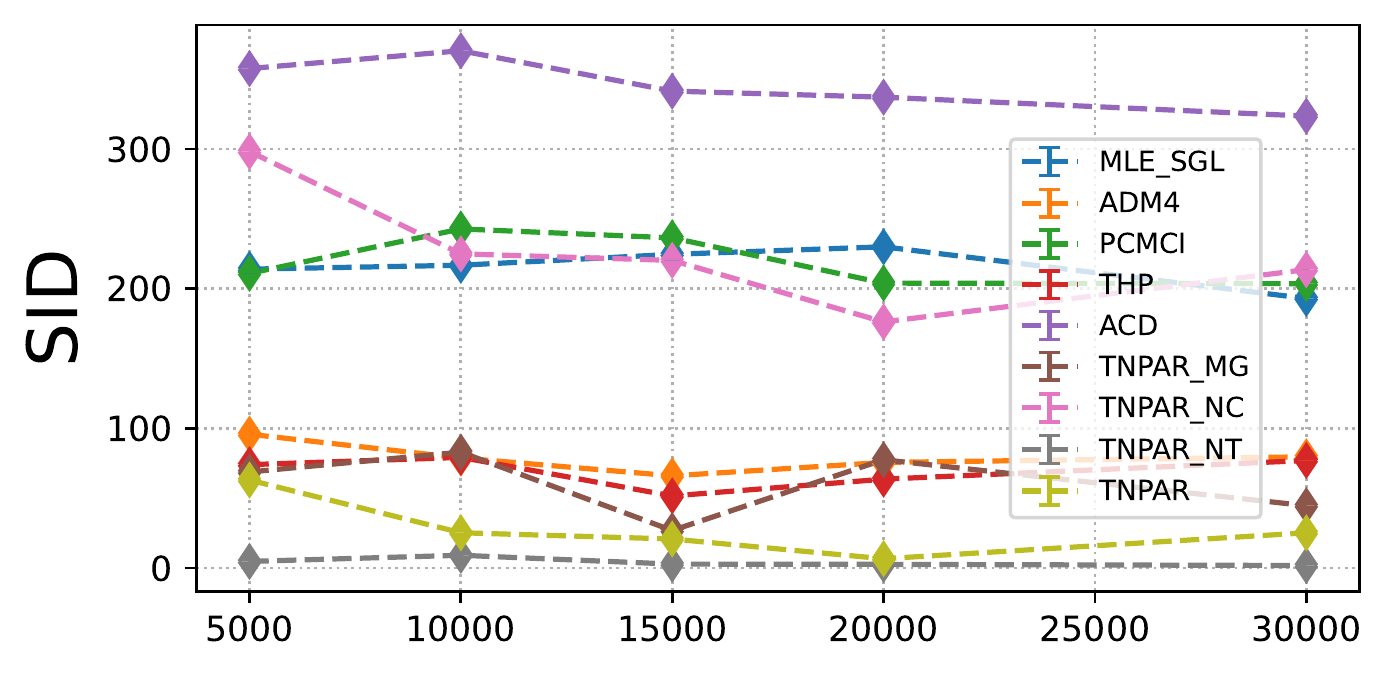}
	\label{appendix fig sid:c}
}
	\subfigure[Sensitivity to $\Delta$]{
	\includegraphics[width=0.45\textwidth]{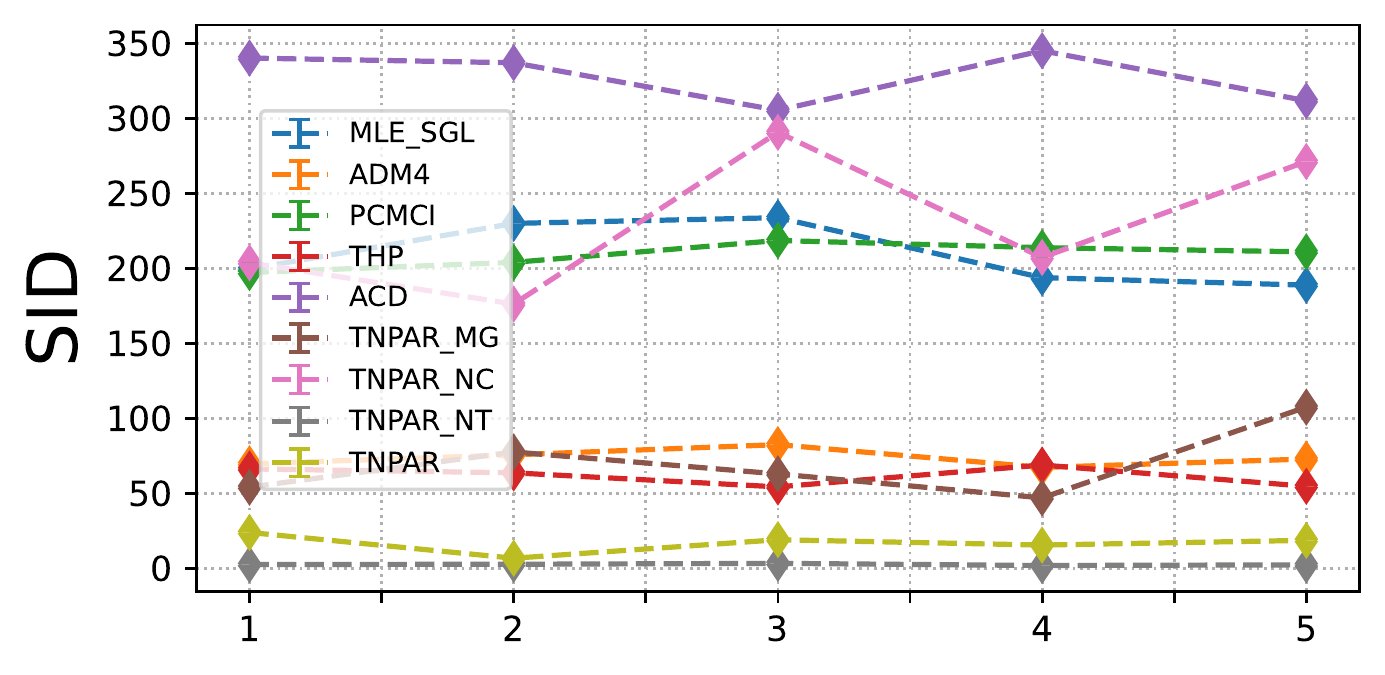}
	\label{appendix fig sid:d}
}
	\subfigure[Sensitivity to Num. of Event Types]{
	\includegraphics[width=0.45\textwidth]{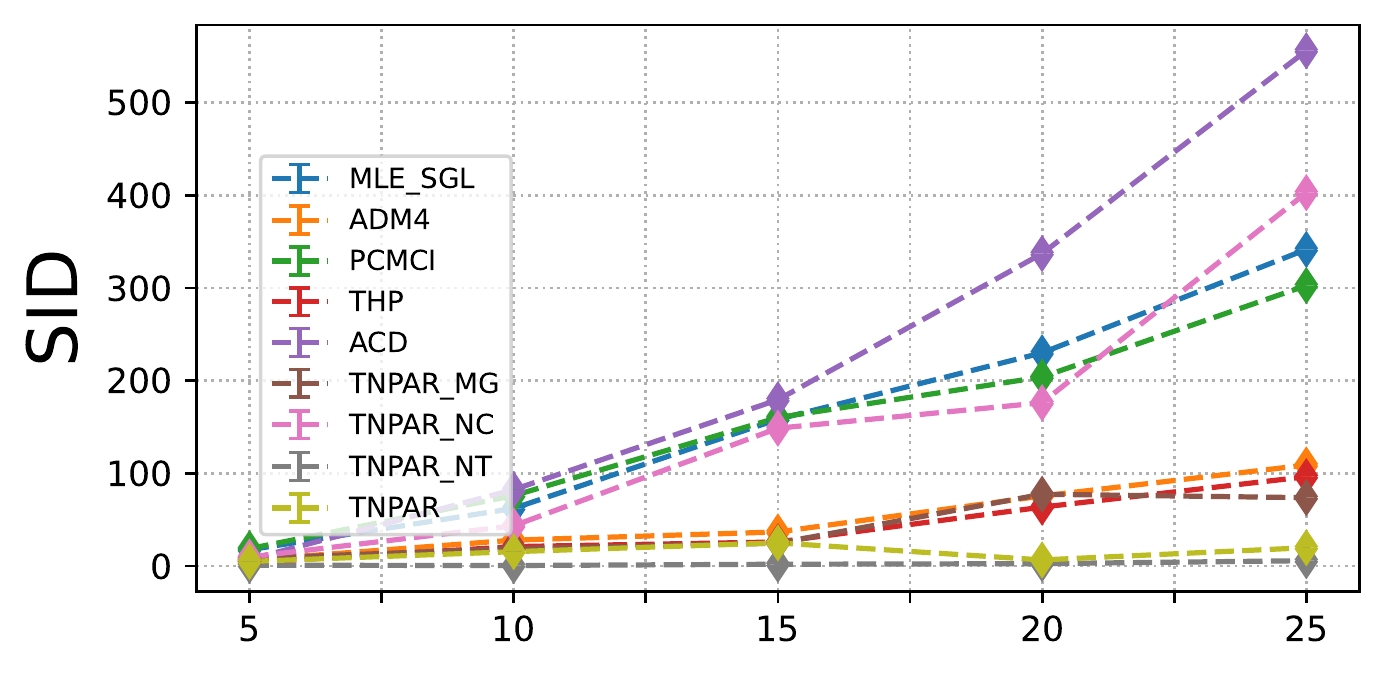}
	\label{appendix fig sid:e}
}
	\subfigure[Sensitivity to Num. of Nodes]{
	\includegraphics[width=0.45\textwidth]{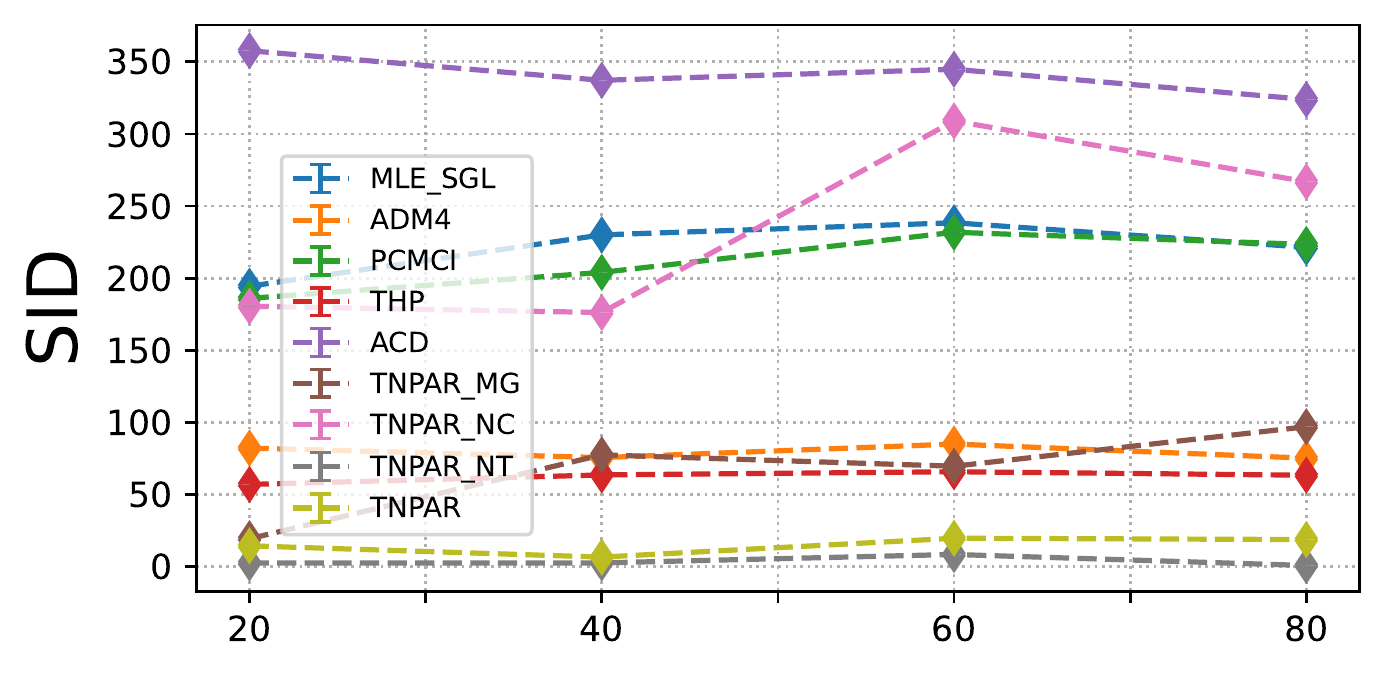}
	\label{appendix fig sid:f}
}
	\caption{SID on the simulated data}	
	\label{appendix fig: sid}
\end{figure*}

\begin{figure*}[]
	\centering
	\subfigure[Sensitivity to Range of $\alpha (\times 1e-2)$]{
		\includegraphics[width=0.45\textwidth]{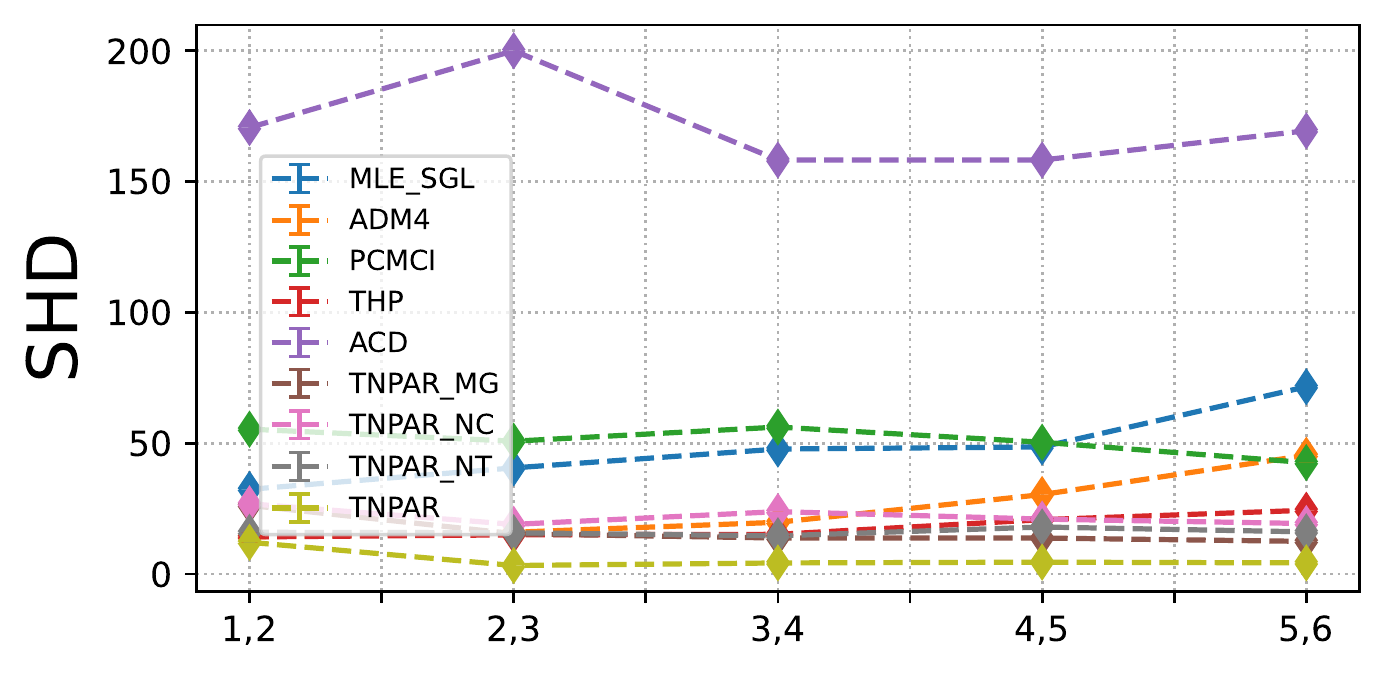}
		\label{appendix fig shd:a}
	}
	\subfigure[Sensitivity to Range of $\mu (\times 1e-5)$]{
	\includegraphics[width=0.45\textwidth]{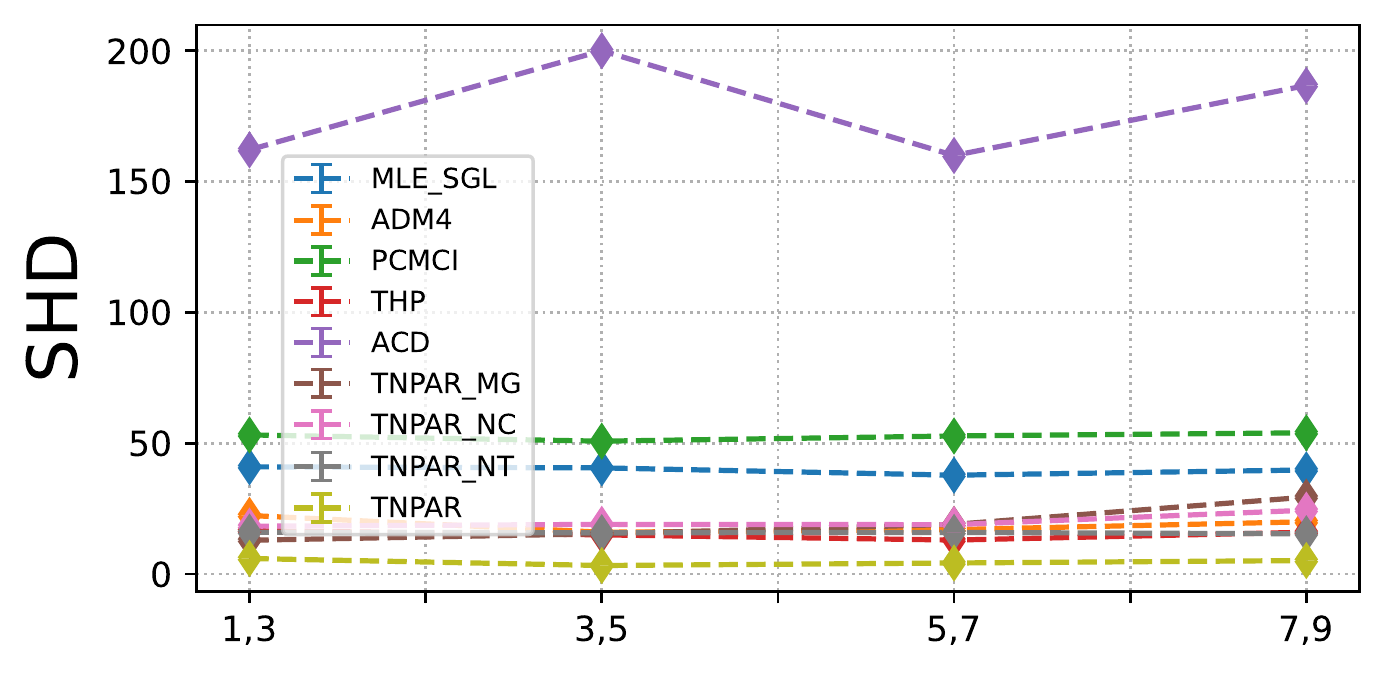}
	\label{appendix fig shd:b}
}
	\subfigure[Sensitivity to Sample Size]{
	\includegraphics[width=0.45\textwidth]{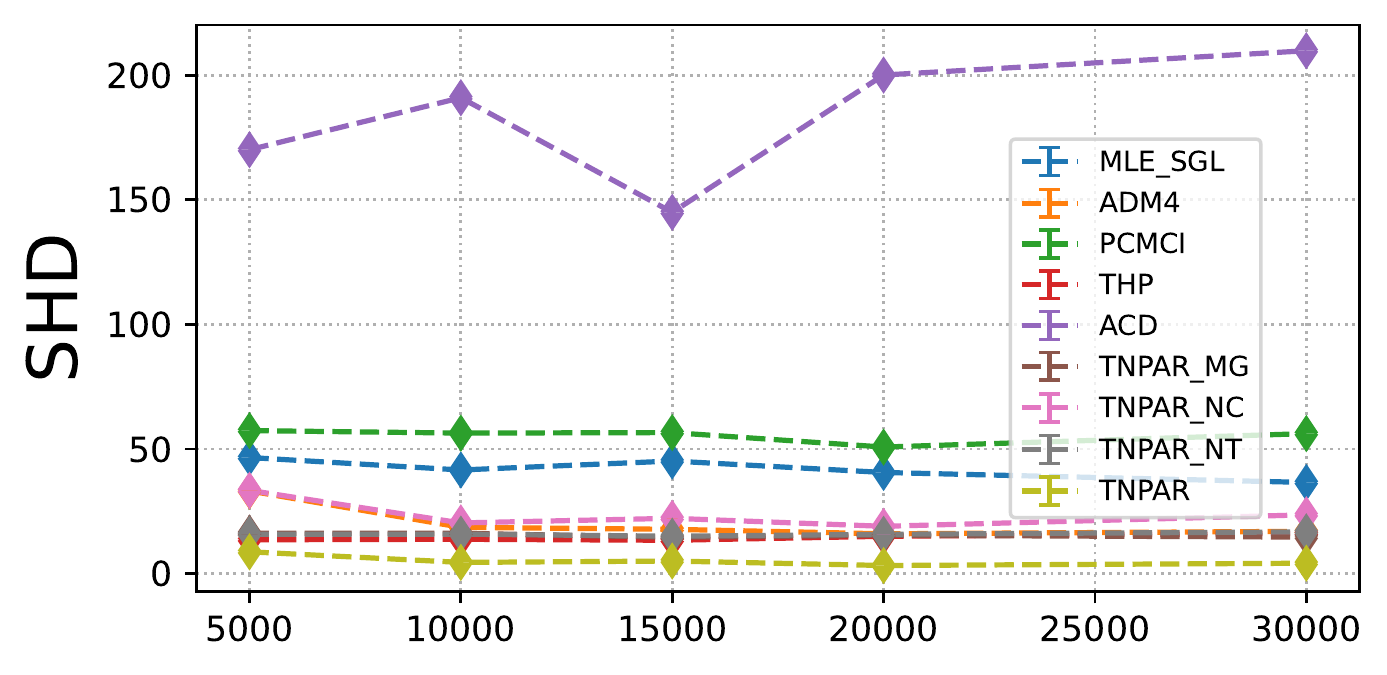}
	\label{appendix fig shd:c}
}
	\subfigure[Sensitivity to $\Delta$]{
	\includegraphics[width=0.45\textwidth]{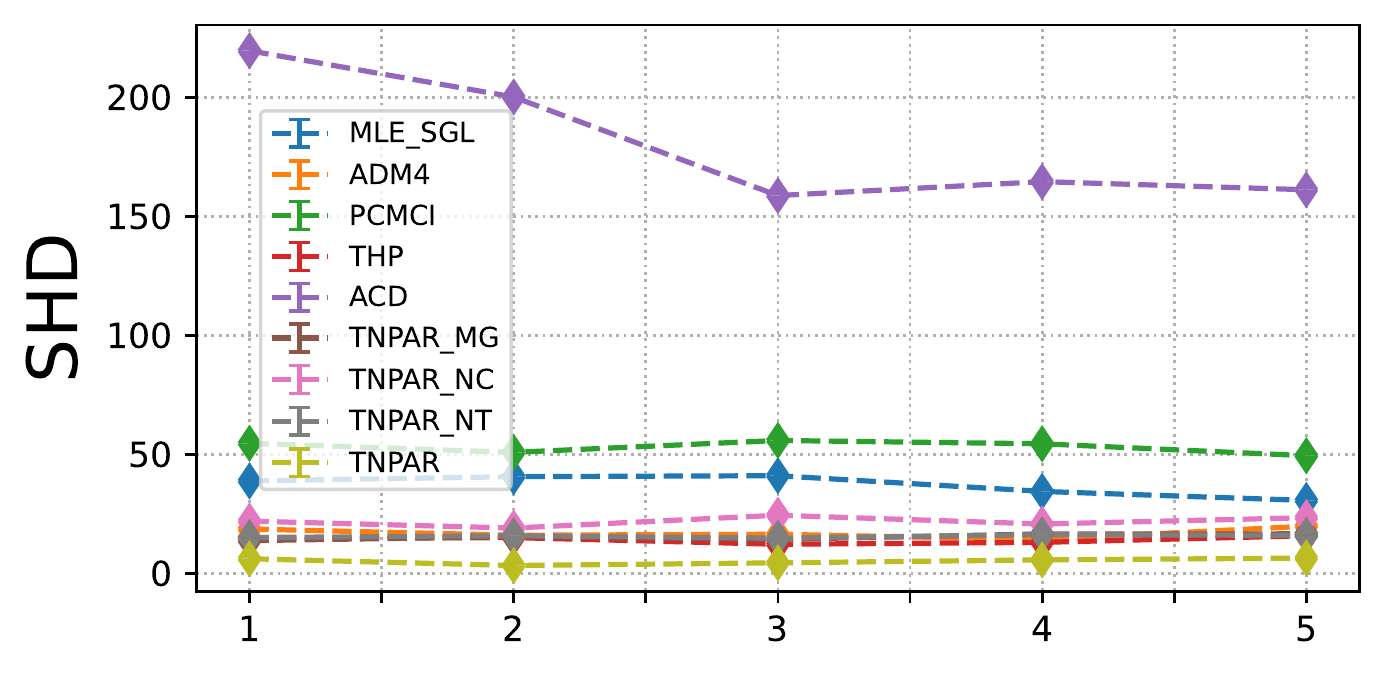}
	\label{appendix fig shd:d}
}
	\subfigure[Sensitivity to Num. of Event Types]{
	\includegraphics[width=0.45\textwidth]{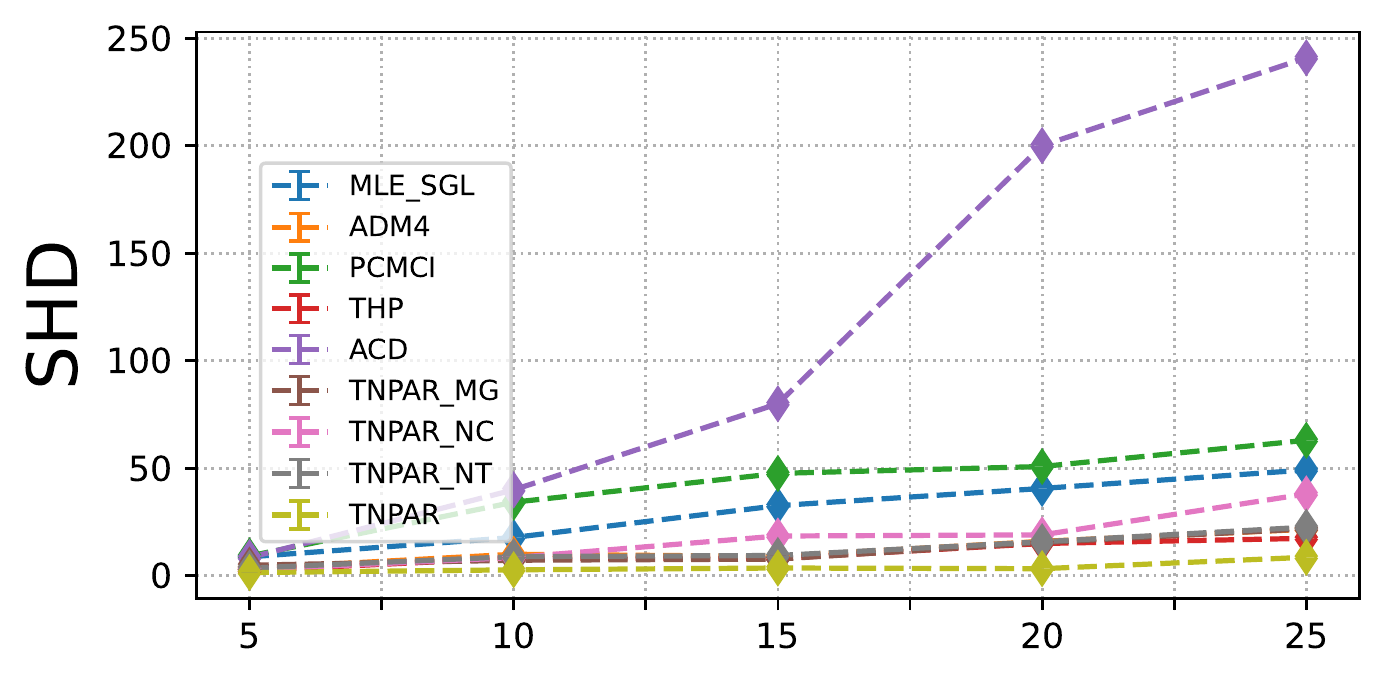}
	\label{appendix fig shd:e}
}
	\subfigure[Sensitivity to Num. of Nodes]{
	\includegraphics[width=0.45\textwidth]{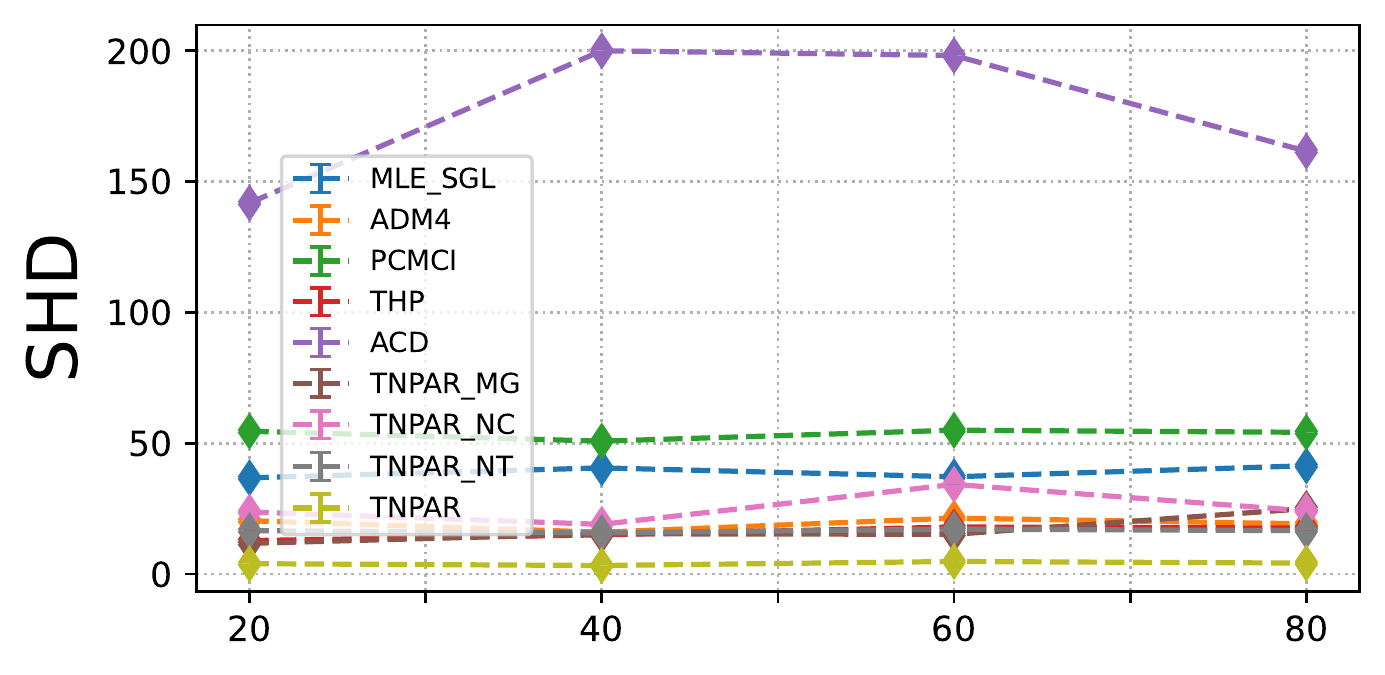}
	\label{appendix fig shd:f}
}
	\caption{SHD on the simulated data}	
	\label{appendix fig: SHD}
\end{figure*}
\end{document}